\documentclass[]{fairmeta}

\usepackage[utf8]{inputenc}
\usepackage[T1]{fontenc} 
\usepackage{hyperref}   
\usepackage{url}        
\usepackage{booktabs}       
\usepackage{amsfonts}       
\usepackage{nicefrac}      
\usepackage{microtype}     
\usepackage{xcolor}        
\usepackage{tikz}
\usepackage{algorithm}
\usepackage[noend]{algpseudocode}
\usepackage{listings}
\usepackage{caption}
\usepackage{amsthm}
\usepackage{amsmath}
\usepackage[inline]{enumitem}
\usepackage{xspace} 
\usepackage{cleveref}
\usepackage{subscript}
\usepackage{fvextra}
\usepackage{framed}
\usepackage{color}
\usepackage{graphicx} 
\usepackage{wrapfig} 
\usepackage{subcaption}
\usepackage{mathtools}
\usepackage{adjustbox}
\usepackage{siunitx}
\usepackage{twemojis}
\usepackage{tabularx}
\usepackage{fontawesome5}
\usepackage{amssymb}
\newtheorem{theorem}{Theorem}
\usepackage[labelsep=colon]{caption}
\definecolor{shadecolor}{rgb}{0.95, 0.95, 0.95}

\definecolor{lightgray}{gray}{0.95}
\definecolor{darkblue}{rgb}{0.1,0.2,0.6}
\definecolor{codegray}{rgb}{0.4,0.4,0.4}

\usetikzlibrary{shapes.geometric, arrows.meta, positioning, fit, calc, backgrounds, shadows}

\definecolor{primaryBlue}{RGB}{0, 105, 148}
\definecolor{workerGreen}{RGB}{46, 139, 87}
\definecolor{evalOrange}{RGB}{230, 126, 34}
\definecolor{outputGold}{RGB}{241, 196, 15}
\definecolor{darkGray}{RGB}{50, 50, 50}



\lstdefinestyle{pddlstyle}{
  numbers=left,
  numbersep=-6pt,
  backgroundcolor=\color{lightgray},
  basicstyle=\ttfamily\scriptsize,
  keywordstyle=\color{darkblue}\bfseries,
  commentstyle=\color{codegray}\itshape,
  showstringspaces=false,
  frame=single,
  breaklines=true,
  language=PDDL
}

\lstdefinestyle{nl}{
  numbers=left,
  numbersep=-6pt,
  backgroundcolor=\color{lightgray},
  basicstyle=\ttfamily\scriptsize,
  keywordstyle=\color{darkblue}\bfseries,
  commentstyle=\color{codegray}\itshape,
  showstringspaces=false,
  frame=single,
  breaklines=true,
  breakindent=0pt, 
  mathescape=true,
  lineskip=-1pt, 
}

\lstdefinestyle{pythonstyle}{
  numbers=none,
  numbersep=-6pt,
  backgroundcolor=\color{lightgray},
  basicstyle=\ttfamily\scriptsize,
  keywordstyle=\color{darkblue}\bfseries,
  commentstyle=\color{codegray}\itshape,
  stringstyle=\color{codegray},
  showstringspaces=false,
  frame=single,
  breaklines=true,
  breakindent=0pt,
  language=Python,
}

\usetikzlibrary{shapes.misc}

\crefname{figure}{Fig.}{Figures}

\newcommand{\airadojo}{{AIRA-dojo}\xspace}

\newcommand{\atlas}{{\textsc{AIRA\textsubscript{2}}}\xspace}
\newcommand{\atlasplus}{{\textsc{AIRA$^{\dagger}_{2}$}}\xspace}
\newcommand{\boldatlas}{{\bfseries\scshape AIRA\textsubscript{2}}\xspace}
\newcommand{\boldatlasplus}{{\bfseries\scshape AIRA$^{\dagger}_{2}$}\xspace}

\newcommand{\percentile}{{Percentile Rank}\xspace}

\title{\atlas: Overcoming Bottlenecks in AI Research Agents}

\author[1,2,*]{Karen Hambardzumyan}
\author[1, *]{Nicolas Baldwin}
\author[1,2,*]{Edan Toledo}
\author[1, *]{Rishi Hazra}
\author[1, *]{Michael Kuchnik}

\author[1]{Bassel Al Omari}
\author[1, 3]{Thomas Simon Foster}
\author[1]{Anton Protopopov}
\author[1]{Jean-Christophe Gagnon-Audet}
\author[1]{Ishita Mediratta}
\author[1]{Kelvin Niu}
\author[1]{Michael Shvartsman}

\author[1, 3]{Alisia Lupidi}
\author[1]{Alexis Audran-Reiss}
\author[1]{Parth Pathak}
\author[1]{Tatiana Shavrina}
\author[1]{Despoina Magka}

\author[1]{Hela Momand}
\author[1]{Derek Dunfield}
\author[1]{Nicola Cancedda}
\author[2]{Pontus Stenetorp}
\author[1]{Carole‑Jean Wu}
\author[1, 3]{Jakob Nicolaus Foerster}
\author[1]{Yoram Bachrach}
\author[1, *]{Martin Josifoski}

\affiliation[1]{FAIR at Meta}
\affiliation[2]{University College London}
\affiliation[3]{University of Oxford}

\contribution[*]{Equal contribution (author order determined by Mario Kart placement)}

\abstract{
Existing research has identified three structural \textit{performance bottlenecks} in AI research agents: (1)~synchronous single-GPU execution constrains \textit{sample throughput}, limiting the benefit of search; (2)~a \textit{generalization gap} where validation-based selection causes overfitting and performance to degrade over extended search horizons; and (3)~the \textit{limited capability} of fixed, single-turn LLM operators imposes a ceiling on search performance. We introduce \atlas, which addresses these bottlenecks through three architectural choices: an asynchronous multi-GPU worker pool that increases experiment throughput linearly; a Hidden Consistent Evaluation protocol that delivers a reliable evaluation signal; and ReAct agents that dynamically scope their actions and debug interactively. On MLE-bench-30, \atlasplus achieves a mean \percentile of \textbf{81.5\%} at 24 hours and \textbf{83.1\%} at 72 hours, outperforming the strongest baseline, which achieves 72.7\%. On AIRS-Bench, \atlas exceeds human state-of-the-art on 6 out of 20 diverse research tasks. Ablations confirm that each architectural component is necessary, that performance follows a predictable scaling law that transfers across LLM backbones, and that the ``overfitting'' reported in prior work was driven by evaluation noise rather than true data memorization.
}
\correspondence{Martin Josifoski at \email{martinjosifoski@meta.com}}

\begin{document}

\maketitle

\section{Introduction}
\label{sec:introduction}

\begin{figure} [!b]
    \centering
    \includegraphics[width=1\linewidth]{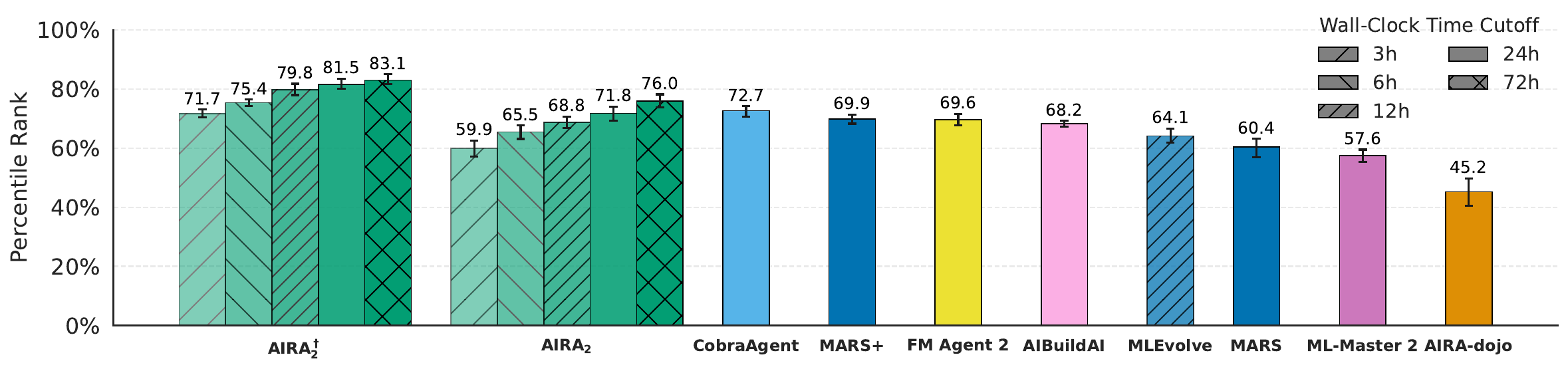}
    \caption{\boldatlas \textbf{performance on MLE-bench-30.} We evaluate \atlas against top-performing agents from the MLE-bench leaderboard across different compute budgets. Utilizing 8 GPU workers for all configurations, \atlas surpasses the strongest baselines at a 24-GPU-hour budget. Performance improves consistently with additional compute, demonstrating the effectiveness of our architectural design. \atlasplus uses Gemini 3.1, while \atlas uses Gemini 3.0.}
    \label{fig:summary}
\end{figure}

The rapid advancement of Large Language Model (LLM) capabilities has enabled the development of autonomous agents capable of executing complex, multi-step workflows~\citep{josifoski2023flows, wu2024autogen}. While these agents have achieved remarkable success in software engineering~\citep{wang2025openhands, anthropic2025claudecode, openai2025codex} and mathematics~\citep{novikov2025alphaevolve, hubert2025olympiad}---benefiting from execution environments with rapid feedback loops of verifiable signal---automating the scientific process presents a distinct class of challenges. Scientific research requires agents to learn through controlled experimentation within a vast, partially observable, open-ended design space~\citep{langley1987scientific}. Unlike coding, research involves navigating noisy evaluation signals and designing proxy tasks to estimate performance~\citep{chan2025mlebench}. 
Furthermore, given the high computational cost and latency of valid experiments, optimizing the research process requires looking beyond individual reasoning capabilities: research agents must be designed to manage long-horizon exploration and parallelize compute-intensive evaluations efficiently.

The best performing open-source agent on MLE-bench\footnote{A challenging benchmark where agents compete in Kaggle competitions.}~\citep{chan2025mlebench}, \airadojo~\citep{toledo2025ai}, frames research agents as a search over candidate solutions, that can be decomposed into search policies and operators.
Through systematic ablations, the authors formalized three structural bottlenecks preventing further scaling:  \textbf{(1)~Compute Throughput}---synchronous single-GPU execution constrains sample generation and limits exploration; \textbf{(2)~Generalization Gap}---validation-test divergence misleads the search signal, causing overfitting over extended research horizons; and \textbf{(3)~Operator Capability}---fixed, single-turn operators limit the agent to shallow, single-turn reasoning that sophisticated search cannot overcome. As noted by \citet{chan2025mlebench,jiang2025aide,zhu2026toward}, these performance plateaus emerge even within a relatively short 24-hour regime, suggesting that addressing these fundamental bottlenecks is a prerequisite for effectively utilizing additional compute.

Guided by these insights, we introduce \atlas, a research agent \emph{designed to overcome these structural bottlenecks}. \atlas resolves the three bottlenecks via specific architectural choices: (1) asynchronous multi-agent exploration, (2) a Hidden Consistent Evaluation protocol, and (3) dynamically scoped ReAct agents. 

First, the standard reliance on synchronous single-GPU execution severely bottlenecks throughput: the search process stalls whenever an expensive experiment runs, starving exploration of samples and limiting the learnings from experiment results within the available wall-clock budget. \emph{\atlas introduces an asynchronous multi-GPU worker pool} and containerization system that decouples decision-making from execution, enabling massively parallel experimentation and increasing experiment throughput linearly with available GPU resources. Concretely, 8 GPUs yield approximately 8$\times$ the experimental throughput, compressing what would otherwise require days of sequential exploration into hours.

Second, the generalization gap undermines long-horizon search: agents optimize validation metrics at the expense of held-out test performance, causing trajectories to overfit. 
\atlas addresses this overfitting with a \emph{Hidden Consistent Evaluation (HCE) protocol}: data splits are standardized once and reused across all candidates, evaluation labels are hidden from the agent, and the search signal is decoupled from the final selection signal. 
HCE leads to improved performance at the 24-hour mark, and enables continued performance gains with additional compute.

Third, operators that depend on fixed prompts and single-turn actions impose a performance ceiling: for example, a static ``debug'' prompt cannot iteratively diagnose complex errors, and no amount of search sophistication can compensate. \emph{\atlas replaces all operators with ReAct agents}~\citep{yao2022react} that autonomously scope their actions---performing exploratory data analysis, running small development experiments, inspecting logs, and more, thus alleviating the limitations of operator design. ReAct agents expand the range of tasks the system can tackle, improve performance, and speed up the discovery of strong solutions, making search more efficient.
    
Together, these design decisions enable \atlas to achieve a mean \percentile of \textbf{81.5\%} at 24 hours and \textbf{83.1\%} at 72 hours on MLE-bench-30~\citep{singh2025openai}, outperforming the strongest baseline, which achieves 72.7\%.
Beyond the Kaggle-style competitions of MLE-bench, we evaluate \atlas in a case study on AIRS-Bench~\citep{lupidi2026airs}, where it exceeds the published state-of-the-art on 6 out of 20 diverse research tasks.
Ablations confirm that each architectural component contributes to performance, and that performance follows a predictable scaling law that transfers across LLM backbones. 
\atlasplus denotes the use of Gemini 3.1, while all other \atlas variants in the paper use Gemini 3.0. 
The consistent gains when using a stronger model backbone confirm that the architecture scales with model capability.

\section{Background}
\label{sec:background}
The domain of automated machine learning has shifted rapidly from simple heuristics~\citep{bergstra2012random,elsken2017simple,li2018hyperband} to autonomous agents capable of long-horizon research~\citep{aiscientistv2}. Recent state-of-the-art systems, such as \textbf{MARS}~\citep{chen2026mars}, \textbf{MLEvolve}~\citep{du2025automlgen}, \textbf{PiEvolve}~\citep{PiEvolve}, \textbf{FM-Agent 2.0}~\citep{li2025fmagent}, and \textbf{ML-Master 2.0}~\citep{liu2025mlmasteraiforaiintegrationexploration}, leverage inference-time scaling and evolutionary search to solve Kaggle competitions. These systems typically model research as a search process over a graph of candidate solutions. However, despite these advances, performance is currently hindered by three structural bottlenecks identified in prior formalizations of the agentic research process~\citep{toledo2025ai}.
\subsection{The Compute \& Throughput Bottleneck}
\label{sec:bg_compute}
Effective search requires high sample throughput---the number of candidate solutions generated and evaluated per unit time---to explore the vast combinatorial space of ML solutions. However, the standard agent architecture often relies on synchronous execution, where the reasoning loop blocks while waiting for experimental feedback.

In the context of MLE-bench, where model training and evaluation can take hours, this serialization is catastrophic. A synchronous agent is effectively ``sample-bound,'' and, on compute-heavy tasks, limited to evaluating only $\approx$1--20 candidates per day. This throughput is insufficient to support the deep exploration required by evolutionary or tree-search methods, rendering them theoretically powerful but practically intractable without parallelization.

\subsection{The Generalization Gap (Overfitting)}
\label{sec:bg_overfitting}
The utility of any search process is contingent on the fidelity of its reward signal. Crucially, this signal need not be a single quantitative metric---in many research settings, such a reduction is neither possible nor desirable. What matters is that the signal, whether numeric, qualitative, or composite, faithfully represents the underlying phenomenon under study. In the competition setting of MLE-bench, this principle manifests concretely as the \textit{generalization gap}---the divergence between the validation metric (used to guide the search) and the held-out test metric (the true objective).

Oracle experiments in prior work revealed that selecting the final submission based on test scores rather than validation scores improves medal rates by \textbf{9--13\%} (absolute)~\citep{toledo2025ai}. This gap stems from two sources: \textbf{(1)} agents ``gaming'' self-reported metrics to satisfy the search objective, and \textbf{(2)} the reuse of the validation set for both hill-climbing (optimization) and final selection, which inevitably leads to overfitting as the search horizon extends. Beyond algorithmic overfitting, the search signal is frequently corrupted by execution noise. Implementation bugs can spuriously inflate validation metrics (see Appendix~\ref{sec:appendix_eval_noise} for a concrete example), while brittle output parsing often leads to missing or erroneous score extraction. Furthermore, stochasticity in data splitting introduces significant variance, allowing inferior solutions to survive selection purely due to favourable random seeds. Closing the generalization gap is therefore a prerequisite for reliable automated research: without a trustworthy reward signal, even a perfect search algorithm will converge on solutions that exploit evaluation artifacts rather than capture genuine predictive structure.

\subsection{The Static Operator Limitation}
\label{sec:bg_operators}

Research agents can be decomposed into a \textit{search policy} (which selects which node to expand) and a set of \textit{atomic operators} (which transform a node into new candidates). In practice, these operators are often hand-designed for anticipated sub-tasks: one prompt for exploratory data analysis, another for feature engineering, another for hyperparameter tuning, and so on. This design is fundamentally brittle: each new domain demands additional human-scoped operators, and the agent can only perform actions its designers anticipated. As tasks grow in complexity---requiring multi-file codebases, shared artifacts, and iterative debugging---a fixed operator pipeline cannot adapt to the unpredictable dependencies that arise.

Empirically, \citet{toledo2025ai} demonstrated that when operators are static (e.g., fixed \textsc{Draft}/\textsc{Improve} prompts), more advanced search algorithms yield no statistically significant improvement over greedy baselines---though this finding is confounded by evaluation noise. Beyond search effectiveness, fixed operators cannot dynamically allocate compute proportional to the difficulty of a sub-problem, limiting the agent to solutions reachable by shallow, single-turn reasoning.

\section{\atlas Research Agent Design}
\label{sec:system}

In this section, we describe the design of \atlas, a system that addresses the three bottlenecks identified in Section~\ref{sec:background}. Its architecture comprises two tiers: a \textbf{Global Orchestrator} that coordinates search over a population of candidates, and an \textbf{Asynchronous Worker Pool of ReAct agents} mutating solutions in isolated containers (Figure~\ref{fig:full_architecture}).

\begin{figure}[t]
    \vspace{-0.5cm}
    \hspace{0cm}
    \includegraphics[width=1.0\textwidth]{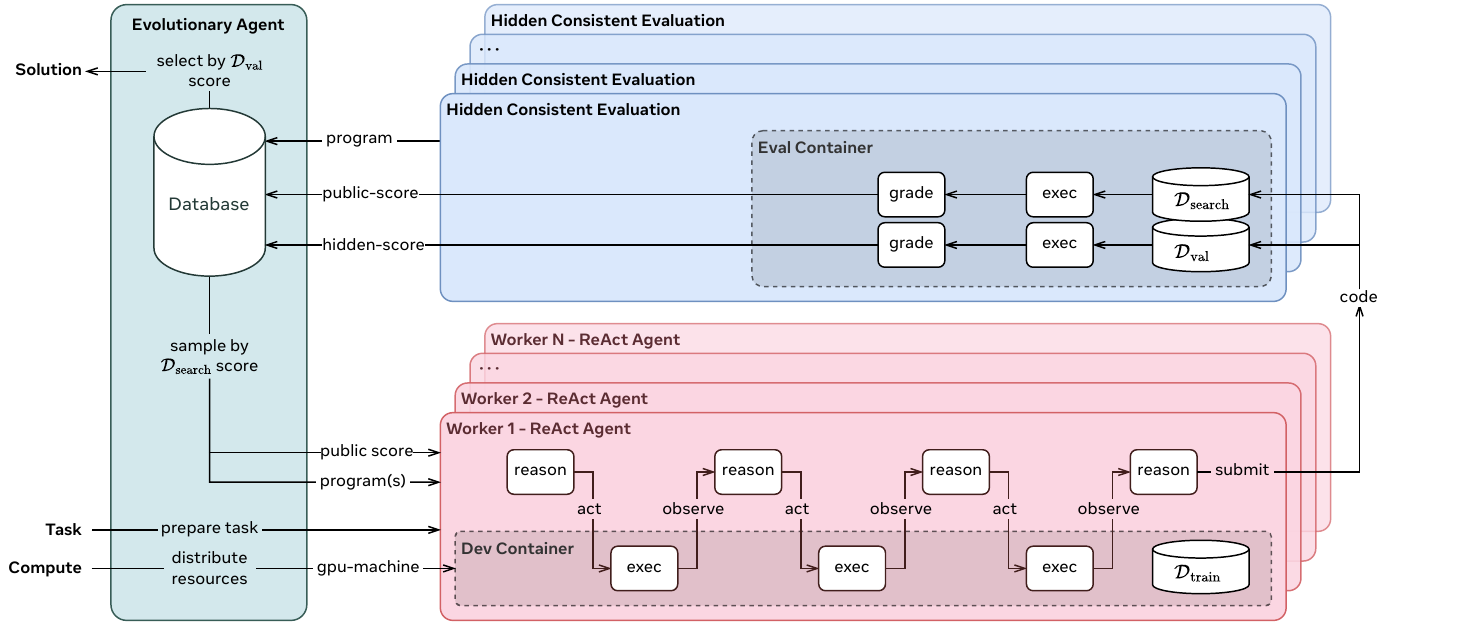}
    \caption{\textbf{\atlas architecture.} ~The Evolutionary Agent orchestrates the search by maintaining a population of candidate solutions and dispatching mutation tasks to the $N$ workers as they become available, without any synchronization barriers. Each worker \textit{asynchronously} executes a ReAct agent which iteratively reasons, executes code, and observes outputs until a candidate solution is ready. Candidate solutions are evaluated in a separate container, and agents observe only the resulting score. Evaluation is partitioned: $\mathcal{D}_{\text{search}}$ guides optimization while $\mathcal{D}_{\text{val}}$ determines final selection. In our main experiments, we use $N$ = 8 workers.}
    \label{fig:full_architecture}
\end{figure}

\subsection{Evolutionary Search}
\label{sec:orchestrator}

The orchestrator maintains a population $\mathcal{P}$ of candidate solutions, their fitness scores, and all other associated metadata. Search proceeds via asynchronous, steady-state evolution~\citep{syswerda1991study}: whenever a worker becomes available, the orchestrator samples a parent (or two) and dispatches a mutation/crossover task.

\paragraph{\textbf{Selection.}} Parents are sampled via temperature-scaled rank-based selection. Given a population $\mathcal{P}$ of size $N$ sorted by fitness, we assign rank $r_i \in \{1, \dots, N\}$ to each individual (where $1$ is the best). The probability of selecting individual $i$ is:
\begin{equation}
    p(i) = \frac{(N - r_i + 1)^{1/T}}{\sum_{j=1}^{N} (N - r_j + 1)^{1/T}},
\end{equation}
where $T$ controls the exploration-exploitation tradeoff. As $T \to 0$, the policy becomes greedy (selecting the rank $1$ individual), while higher $T$ increases diversity. We opted for rank-based rather than fitness-proportionate selection because ranks are invariant to the magnitude and scale of fitness scores, which vary widely across tasks.

\paragraph{\textbf{Mutation/Crossover.}} The orchestrator randomly selects between \textit{mutation} (refining a single parent) and \textit{crossover} (combining two parents) with probability $c$. Both operations are executed by ReAct agents (Section~\ref{sec:react_agents}), which receive the parent solution(s) (and other metadata) as context.

\subsection{Scaling Compute: Asynchronous Multi-GPU Execution}
\label{sec:async_execution}

The asynchronous nature of steady-state evolution inherently facilitates parallel orchestration, allowing us to distribute the search workload across concurrent workers without synchronization barriers. Unlike generational evolution, which must wait for all workers to complete before proceeding, steady-state evolution allows the orchestrator to sample parents and dispatch mutations as soon as any individual worker becomes available. This is particularly beneficial when mutations are longer, more involved processes---such as multi-step ReAct trajectories that may vary widely in duration---since fast-completing workers are never left idle waiting for slower ones. To bypass the complexity of dynamic resource scheduling, we adopt a static allocation scheme where each worker is assigned a dedicated GPU. This hardware configuration is mirrored across both development and evaluation environments, ensuring that every execution begins from an identical, clean slate.

\paragraph{\textbf{Remote Execution \& Containerization.}} We employ a remote tool execution system to decouple agent logic from the execution environment. Workers execute code within ephemeral \texttt{Apptainer}~\citep{gregory_m_kurtzer_2021_4667718} containers using the \texttt{Superimage} environment~\citep{toledo2025ai}, which comes pre-installed with a comprehensive suite of Python, machine learning, and data science packages. Crucially, containers run in \texttt{fakeroot} mode, granting the agent perceived root privileges to install system-level dependencies via \texttt{apt} or \texttt{pip}. This setup ensures a flexible, reproducible, and robust action space where crashed containers do not affect the orchestrator or other workers.

\paragraph{\textbf{Stateful Interaction.}} Unlike previous systems such as AIDE~\citep{jiang2025aide} and \airadojo, our Bash command and Jupyter kernel execution tools are stateful. This allows agents to maintain context across multiple turns, enabling a more interactive and iterative debugging process. Tool outputs also include execution duration, allowing agents to monitor their own efficiency.

\paragraph{\textbf{Resource Allocation.}} We enforce a strict 1:1 worker-to-GPU mapping using NVIDIA H200 GPUs (141GB VRAM). Each worker is allocated 12 logical CPU cores and a dedicated 120GB of system RAM, providing sufficient compute for training large models, running hyperparameter sweeps, or building deep ensembles without resource contention.

\paragraph{\textbf{Data Management \& Limits.}} The program database resides in-memory for fast access, with large artifacts automatically offloaded to hard disk. Subagents communicate exclusively through this central database, contributing ideas and code asynchronously. To optimize resource utilization, evaluation is performed in the foreground without a separate job queue, running in an identical container with a dedicated GPU. The search process continues until a global hard limit is reached, with individual code executions capped at a 9-hour hard time limit to prevent stalled processes.

\subsection{Closing the Generalization Gap: Hidden Consistent Evaluation}
\label{sec:hidden_eval}

To close the generalization gap, \atlas decouples the signal used for search from the signal used for final selection, and externalizes all evaluation to prevent metric gaming.
Beyond serving as a practical safeguard, this protocol is designed as an experimental tool: by controlling the evaluation procedure, we can systematically test whether agents truly overfit to their data or whether previously reported degradation stems from evaluation noise (Section~\ref{sec:ablation_evaluation}).

\textbf{Data partitioning.} Before search begins, available labels are split into three disjoint sets:
\begin{itemize}[leftmargin=1.5em, itemsep=0.3ex]
    \item $\mathcal{D}_{\text{train}}$: visible to the agent for model training.
    \item $\mathcal{D}_{\text{search}}$: used by the orchestrator for fitness computation; \textit{labels hidden from agents}.
    \item $\mathcal{D}_{\text{val}}$: used \textit{only} for final selection after search terminates; \textit{hidden from both agents and the search process.}
\end{itemize}
These splits are created once via random sampling of the available labeled training data (80\%/10\%/10\%) and reused identically across all seeds. Notably, baselines are evaluated using their published results on the test set with access to all training data, while \atlas reserves 20\% of training data for search/selection, making the comparison conservative. Upon submission, the solution submits predictions for $\mathcal{D_{\text{test}}}$. Thus, all splits are consistent across programs and seeds.

\textbf{Externalized evaluation.} Agents never self-report metrics. When a worker returns a solution, the orchestrator evaluates it on $\mathcal{D}_{\text{search}}$ in a separate container, mirroring the dev environment. Agents observe only the resulting score, not the labels, preventing feedback loops that enable metric hacking. The evaluation is performed on all $\mathcal{D}_{\text{search}}$, $\mathcal{D}_{\text{val}}$ and $\mathcal{D}_{\text{test}}$ splits, but $\mathcal{D}_{\text{test}}$ is neither forwarded to agents nor to orchestrator.

\textbf{Decoupled selection.} Because $\mathcal{D}_{\text{val}}$ is never used during search, the final selection is insulated from hill-climbing dynamics. This separation---combined with the hidden consistent evaluation procedures across all candidates---is designed to reduce the validation--test gap observed in prior systems.

\subsection{Beyond Static Operators: ReAct Agents}
\label{sec:react_agents}

To overcome the operator limitations, \atlas replaces static operators with ReAct agents~\citep{yao2022react} that execute multi-step reasoning trajectories. Each mutation produces a trajectory:
\begin{equation}
    \tau = \bigl(\text{Reason}_1, \text{Act}_1, \text{Obs}_1, \ldots, \text{Reason}_{K-1}, \text{Act}_{K-1}, \text{Obs}_{K-1}, \text{Reason}_{K},\text{Act}_K\bigr),
\end{equation}
where \textit{Actions} ($\text{Act}_t$) are Python or Bash commands executed in a sandboxed environment, and \textit{Observations} ($\text{Obs}_t$) consist of execution outputs. Within the ReAct trajectory, no additional guidance and instructions are provided. The final $\text{Act}_K$ is the ``submit'' tool, which sends the solution back to the orchestrator; the orchestrator then performs the evaluation and finally adds the program and the artifacts to the database.

This formulation provides two capabilities that fixed operators lack:

\textbf{Dynamic scoping.} The agent decides at runtime what actions are necessary. For tasks requiring exploratory data analysis, it inspects distributions and observes correlations before modelling; conversely, for tasks focused on model refinement, it prioritizes hyperparameter tuning and architecture evaluation. It has the ability to do local experimentation and simulation before committing to specific ideas. This eliminates brittle ``scope engineering'' in fixed operator prompts.

\textbf{Interactive debugging.} When code raises an exception, the agent observes the traceback within the same trajectory, hypothesizes a fix, and re-executes---resolving errors without forfeiting the mutation attempt. Static \textsc{Debug} operators, by contrast, lack iterative access to the execution environment and require handcrafted re-prompting and consecutive operator calls.

\section{Experiments and Results}
\label{sec:main_results}

\subsection{Experimental Setup}
\label{sec:setup}

\paragraph{\textbf{Tasks.}} 
We evaluate \atlas on MLE-bench-30, a curated subset of 30 Kaggle competitions from MLE-bench~\citep{chan2025mlebench} used in the GPT-5 system card~\citep{singh2025openai}. We opted for this subset to facilitate a lightweight yet representative evaluation; unlike MLE-bench Lite, which focuses primarily on low-complexity tasks, MLE-bench-30 spans a broader difficulty spectrum stratified into 5 low, 20 medium, and 5 high complexity tasks. Following standard protocol, we report the \textbf{medal rate}---the fraction of runs achieving at least a bronze medal on the Kaggle leaderboard. However, our primary analytical metric is \textbf{\percentile}, representing the agent's simulated leaderboard position. Calculated as $P = \frac{N - R}{N - 1} \times 100$ (where $N$ is total entries and $R$ is the ordinal rank, with $1$ being best), \percentile offers three advantages over medal rate: (1) it is continuous rather than discrete, enabling finer-grained comparisons; (2) it captures the full distribution of performance rather than a binary medal outcome; and (3) it avoids threshold effects near medal boundaries that amplify noise in aggregate statistics~\citep{audran2025does}. Importantly, gains at higher percentiles are progressively harder to achieve, as each increment requires outperforming increasingly skilled human competitors.

\textbf{System configuration.} While \atlas can scale to large GPU counts, for these experiments, we use 8$\times$ \texttt{NVIDIA H200 GPUs} (\texttt{141GB VRAM} each) with 1:1 worker-to-GPU mapping. 
Workers execute in \texttt{Apptainer} containers with \texttt{CUDA}, \texttt{PyTorch}, and standard data science libraries. 
ReAct agents are powered by \texttt{Gemini 3.0 Pro Preview}~\citep{google2025gemini3} in all \atlas variants, with the exception of \atlasplus, which shares the same architecture but uses \texttt{Gemini 3.1 Pro Preview} as the LLM backbone.
The orchestrator runs steady-state evolution with temperature-scaled rank selection ($T=0.2$) and crossover probability $p_c=15\%$.

\textbf{Protocol.} Each task runs for 72 hours wall-clock time. We run 3 independent seeds per task and report mean $\pm$ SE. Task data is partitioned as 80\% $\mathcal{D}_{\text{train}}$ / 10\% $\mathcal{D}_{\text{search}}$ / 10\% $\mathcal{D}_{\text{val}}$. Final submissions are evaluated directly on the held-out test set $\mathcal{D_{\text{test}}}$ \textbf{without} retraining on the full available training data $\cup\{\mathcal{D_{\text{train}}}, \mathcal{D_{\text{search}}}, \mathcal{D_{\text{val}}}\}$.

\begin{table}[t]
\centering
\small
\renewcommand{\arraystretch}{1.75}
\begin{tabularx}{\textwidth}{p{4.1cm} ccc ccc ccc ccc}
\toprule
& \multicolumn{3}{c}{\textbf{\% \percentile}} & \multicolumn{3}{c}{\textbf{\% Bronze+}} & \multicolumn{3}{c}{\textbf{\% Silver+}} & \multicolumn{3}{c}{\textbf{\% Gold}} \\
\cmidrule(lr){2-4} \cmidrule(lr){5-7} \cmidrule(lr){8-10} \cmidrule(lr){11-13}
Method & 3h & 24h & 72h & 3h & 24h & 72h & 3h & 24h & 72h & 3h & 24h & 72h \\
\midrule
\boldatlasplus & \textbf{\shortstack{71.7 \\ \scriptsize $\pm$3.5}} & \textbf{\shortstack{81.5 \\ \scriptsize $\pm$3.2}} & \textbf{\shortstack{83.1 \\ \scriptsize $\pm$3.2}} & \textbf{\shortstack{54.4 \\ \scriptsize $\pm$5.3}} & \textbf{\shortstack{72.2 \\ \scriptsize $\pm$4.7}} & \textbf{\shortstack{76.7 \\ \scriptsize $\pm$4.5}} & \textbf{\shortstack{47.8 \\ \scriptsize $\pm$5.3}} & \textbf{\shortstack{65.6 \\ \scriptsize $\pm$5.0}} & \textbf{\shortstack{73.3 \\ \scriptsize $\pm$4.7}} & \textbf{\shortstack{28.9 \\ \scriptsize $\pm$4.8}} & \textbf{\shortstack{41.1 \\ \scriptsize $\pm$5.2}} & \textbf{\shortstack{52.2 \\ \scriptsize $\pm$5.3}} \\
\atlas & \shortstack{59.9 \\ \scriptsize $\pm$3.6} & \shortstack{71.8 \\ \scriptsize $\pm$3.5} & \shortstack{76.0 \\ \scriptsize $\pm$3.4} & \shortstack{38.9 \\ \scriptsize $\pm$5.2} & \shortstack{57.8 \\ \scriptsize $\pm$5.2} & \shortstack{61.1 \\ \scriptsize $\pm$5.2} & \shortstack{33.3 \\ \scriptsize $\pm$5.0} & \shortstack{50.0 \\ \scriptsize $\pm$5.3} & \shortstack{58.9 \\ \scriptsize $\pm$5.2} & \shortstack{20.0 \\ \scriptsize $\pm$4.2} & \shortstack{32.2 \\ \scriptsize $\pm$5.0} & \shortstack{36.7 \\ \scriptsize $\pm$5.1} \\
\atlas (4 GPU) & \shortstack{56.9 \\ \scriptsize $\pm$3.6} & \shortstack{71.2 \\ \scriptsize $\pm$3.4} & \shortstack{76.5 \\ \scriptsize $\pm$3.4} & \shortstack{37.8 \\ \scriptsize $\pm$5.1} & \shortstack{55.6 \\ \scriptsize $\pm$5.3} & \shortstack{62.2 \\ \scriptsize $\pm$5.1} & \shortstack{30.0 \\ \scriptsize $\pm$4.9} & \shortstack{47.8 \\ \scriptsize $\pm$5.3} & \shortstack{55.6 \\ \scriptsize $\pm$5.3} & \shortstack{11.1 \\ \scriptsize $\pm$3.3} & \shortstack{30.0 \\ \scriptsize $\pm$4.9} & \shortstack{40.0 \\ \scriptsize $\pm$5.2} \\
\atlas (1 GPU) & \shortstack{41.3 \\ \scriptsize $\pm$3.9} & \shortstack{56.8 \\ \scriptsize $\pm$3.8} & \shortstack{63.5 \\ \scriptsize $\pm$3.8} & \shortstack{22.2 \\ \scriptsize $\pm$4.4} & \shortstack{41.1 \\ \scriptsize $\pm$5.2} & \shortstack{51.1 \\ \scriptsize $\pm$5.3} & \shortstack{17.8 \\ \scriptsize $\pm$4.1} & \shortstack{32.2 \\ \scriptsize $\pm$5.0} & \shortstack{41.1 \\ \scriptsize $\pm$5.2} & \shortstack{10.0 \\ \scriptsize $\pm$3.2} & \shortstack{20.0 \\ \scriptsize $\pm$4.2} & \shortstack{24.4 \\ \scriptsize $\pm$4.6} \\
\atlas (No Subagents) & \shortstack{54.4 \\ \scriptsize $\pm$3.7} & \shortstack{68.6 \\ \scriptsize $\pm$3.6} & \shortstack{73.7 \\ \scriptsize $\pm$3.6} & \shortstack{33.3 \\ \scriptsize $\pm$5.0} & \shortstack{52.2 \\ \scriptsize $\pm$5.3} & \shortstack{60.0 \\ \scriptsize $\pm$5.2} & \shortstack{30.0 \\ \scriptsize $\pm$4.9} & \shortstack{47.8 \\ \scriptsize $\pm$5.3} & \shortstack{53.3 \\ \scriptsize $\pm$5.3} & \shortstack{13.3 \\ \scriptsize $\pm$3.6} & \shortstack{32.2 \\ \scriptsize $\pm$5.0} & \shortstack{37.8 \\ \scriptsize $\pm$5.1} \\
\atlas (No HCE) & \shortstack{43.4 \\ \scriptsize $\pm$3.8} & \shortstack{56.8 \\ \scriptsize $\pm$4.2} & \shortstack{56.3 \\ \scriptsize $\pm$4.3} & \shortstack{29.7 \\ \scriptsize $\pm$4.6} & \shortstack{46.5 \\ \scriptsize $\pm$5.0} & \shortstack{47.5 \\ \scriptsize $\pm$5.0} & \shortstack{25.7 \\ \scriptsize $\pm$4.4} & \shortstack{45.5 \\ \scriptsize $\pm$5.0} & \shortstack{46.5 \\ \scriptsize $\pm$5.0} & \shortstack{12.9 \\ \scriptsize $\pm$3.3} & \shortstack{30.7 \\ \scriptsize $\pm$4.6} & \shortstack{32.7 \\ \scriptsize $\pm$4.7} \\
\atlas (No Evo.) & \shortstack{54.7 \\ \scriptsize $\pm$3.6} & \shortstack{64.0 \\ \scriptsize $\pm$3.5} & \shortstack{65.2 \\ \scriptsize $\pm$3.5} & \shortstack{35.3 \\ \scriptsize $\pm$5.2} & \shortstack{45.9 \\ \scriptsize $\pm$5.4} & \shortstack{47.1 \\ \scriptsize $\pm$5.4} & \shortstack{31.8 \\ \scriptsize $\pm$5.1} & \shortstack{40.0 \\ \scriptsize $\pm$5.3} & \shortstack{43.5 \\ \scriptsize $\pm$5.4} & \shortstack{16.5 \\ \scriptsize $\pm$4.0} & \shortstack{23.5 \\ \scriptsize $\pm$4.6} & \shortstack{24.7 \\ \scriptsize $\pm$4.7} \\
\midrule
\multicolumn{2}{l}{CobraAgent~\citep{dalpha2026cobraagent}}  & \shortstack{72.7 \\ \scriptsize $\pm$0.7} &  &  & \shortstack{78.9 \\ \scriptsize $\pm$1.1} &  &  & \shortstack{53.3 \\ \scriptsize $\pm$3.8} &  &  & \shortstack{16.7 \\ \scriptsize $\pm$3.3} &  \\
\multicolumn{2}{l}{MARS+~\citep{chen2026mars}} & \shortstack{69.9 \\ \scriptsize $\pm$0.2} &  &  & \shortstack{64.4 \\ \scriptsize $\pm$1.1} &  &  & \shortstack{51.1 \\ \scriptsize $\pm$2.9} &  &  & \shortstack{24.4 \\ \scriptsize $\pm$2.2} &  \\
\multicolumn{2}{l}{FM-Agent 2.0~\citep{li2025fmagent}} & \shortstack{69.6 \\ \scriptsize $\pm$2.2} &  &  & \shortstack{61.1 \\ \scriptsize $\pm$2.9} &  &  & \shortstack{57.8 \\ \scriptsize $\pm$2.9} &  &  & \shortstack{36.7 \\ \scriptsize $\pm$3.3} &  \\
\multicolumn{2}{l}{AIBuildAI~\citep{zhang2026aibuildai}} & \shortstack{68.2 \\ \scriptsize $\pm$0.5} &  &  & \shortstack{64.4 \\ \scriptsize $\pm$1.1} &  &  & \shortstack{42.2 \\ \scriptsize $\pm$2.2} &  &  & \shortstack{12.2 \\ \scriptsize $\pm$1.1} &  \\
\multicolumn{2}{l}{MLEvolve~\citep{du2025automlgen}} & \shortstack{64.1 \\ \scriptsize $\pm$0.3} &  &  & \shortstack{57.8 \\ \scriptsize $\pm$2.9} &  &  & \shortstack{52.2 \\ \scriptsize $\pm$1.1} &  &  & \shortstack{22.2 \\ \scriptsize $\pm$4.8} &  \\
\multicolumn{2}{l}{MARS~\citep{chen2026mars}} & \shortstack{60.4 \\ \scriptsize $\pm$3.1} &  &  & \shortstack{54.4 \\ \scriptsize $\pm$4.0} &  &  & \shortstack{44.4 \\ \scriptsize $\pm$4.0} &  &  & \shortstack{18.9 \\ \scriptsize $\pm$1.1} &  \\
\multicolumn{2}{l}{ML-Master 2.0~\citep{liu2025mlmasteraiforaiintegrationexploration}} & \shortstack{57.6 \\ \scriptsize $\pm$1.2} &  &  & \shortstack{52.2 \\ \scriptsize $\pm$4.0} &  &  & \shortstack{40.0 \\ \scriptsize $\pm$5.8} &  &  & \shortstack{8.9 \\ \scriptsize $\pm$1.1} &  \\
\multicolumn{2}{l}{PiEvolve~\citep{PiEvolve}} & \shortstack{54.1 \\ \scriptsize $\pm$1.6} &  &  & \shortstack{54.4 \\ \scriptsize $\pm$1.1} &  &  & \shortstack{50.0 \\ \scriptsize $\pm$1.9} &  &  & \shortstack{27.8 \\ \scriptsize $\pm$5.6} &  \\
\multicolumn{2}{l}{AIRA-dojo~\citep{toledo2025ai}} & \shortstack{39.5 \\ \scriptsize $\pm$0.7} &  &  & \shortstack{25.8 \\ \scriptsize $\pm$1.3} &  &  & \shortstack{20.5 \\ \scriptsize $\pm$1.2} &  &  & \shortstack{8.8 \\ \scriptsize $\pm$0.7} &  \\
\bottomrule
\end{tabularx}
\caption{\textbf{Main performance evaluation across varying time budgets (3h, 24h, 72h).} We report \percentile metric and medal rates (Bronze+, Silver+, and Gold) for \atlas (8 GPUs, Subagents, Hidden Consistent Eval (HCE)) against ablations (4 GPU, 1 GPU, No Subagents, No HCE) and state-of-the-art baselines. Confidence intervals denote $\pm \text{SE}.$ \atlasplus uses Gemini 3.1; all other configurations use Gemini 3.0. \atlasplus surpasses all baselines at 24h, and both \atlas and \atlasplus continue to improve through 72h, demonstrating the effectiveness of the proposed architecture.}
\label{tab:main-results}
\end{table}

\subsection{Main Results}
\label{sec:perf}

Table~\ref{tab:main-results} presents the performance of \atlas and \atlasplus under varying time budgets and component ablations. Baselines span a range of LLM backbones: MARS+, MARS, MLEvolve, FM-Agent 2.0, PiEvolve, and AIRA-dojo use Gemini 3.0 Pro Preview~\citep{google2025gemini3}; CobraAgent~\citep{dalpha2026cobraagent} uses a Gemini 3.1 Pro and Flash ensemble; AIBuildAI~\citep{zhang2026aibuildai} uses Claude Opus 4.6; and ML-Master 2.0 uses DeepSeek V3.2-Speciale~\citep{liu2025deepseek}. We note that CobraAgent, AIBuildAI, MARS+, MARS, FM-Agent 2.0, and MLEvolve are concurrent work released after the development and evaluation of \atlas. MARS+ and ML-Master 2.0 utilize 2 GPUs; all other baselines use a single GPU. We report \percentile as our primary metric, as discrete medal thresholds are sensitive to noise near boundaries and fail to capture progress on difficult tasks where no agent achieves medals---for instance, improving from the 5th to the 55th percentile on a hard task is invisible to medal rates but reflected in \percentile.

\paragraph{\textbf{Early Search (3h)}} At the 3-hour mark, \atlasplus achieves a \percentile of 71.7\%, already matching the strongest 24-hour baselines such as CobraAgent (72.7\%) and MARS+ (69.9\%) within just 24 GPU-hours. The base \atlas configuration reaches 59.9\%, comparable to top single-GPU agents. While the system is not optimized for compute efficiency---it uses 8 GPUs to prioritize search breadth over resource efficiency---these early results demonstrate that strong initial solutions emerge quickly and provide a foundation for continued refinement.

\paragraph{\textbf{Standard Evaluation (24h)}} At 24 hours, \atlasplus achieves a mean \percentile of 81.5\%, surpassing the strongest baseline, CobraAgent~\citep{dalpha2026cobraagent}, by 8.8 percentage points. The base \atlas configuration reaches 71.8\%, competitive with the top baselines; notably, a 4-subagent configuration (each with a dedicated GPU) achieves 71.2\%---within 0.6 percentage points---indicating that competitive performance does not require the full 8-subagent setup (see Section~\ref{sec:scaling-laws}). Both results are achieved despite reserving 20\% of the training data for the Hidden Consistent Evaluation protocol (Section~\ref{sec:hidden_eval}), which trades short-term performance for a more reliable search signal.

\paragraph{\textbf{Long-Horizon Search (72h)}} \atlas is designed for sustained improvement over extended time horizons. At 72 hours, \atlasplus reaches 83.1\% \percentile and \atlas reaches 76.0\%---both surpassing all reported 24-hour baselines. Notably, unlike prior systems where extended runtime leads to performance degradation due to overfitting~\citep{toledo2025ai}, both configurations continue to improve with additional compute. The consistent gains when using a stronger model backbone confirm that the architecture scales with model capability. We analyse the specific contributions of each ablated component (Subagents, Consistent Eval) in Section~\ref{sec:ablation}.
\subsection{Ablation Studies}
\label{sec:ablation}

We adopt a subtractive ablation design, removing one component at a time from the full system, to directly test whether each is necessary for the observed performance. All ablations use \atlas (Gemini 3.0 Pro Preview) as the base configuration.

\subsubsection{\textbf{Resolving the Compute Bottleneck}}
\label{sec:ablation_compute}

\begin{figure}[t]
\centering
\begin{subfigure}[t]{0.48\textwidth}
\centering
\hspace{-1.5cm}
\includegraphics[width=1.14\textwidth]{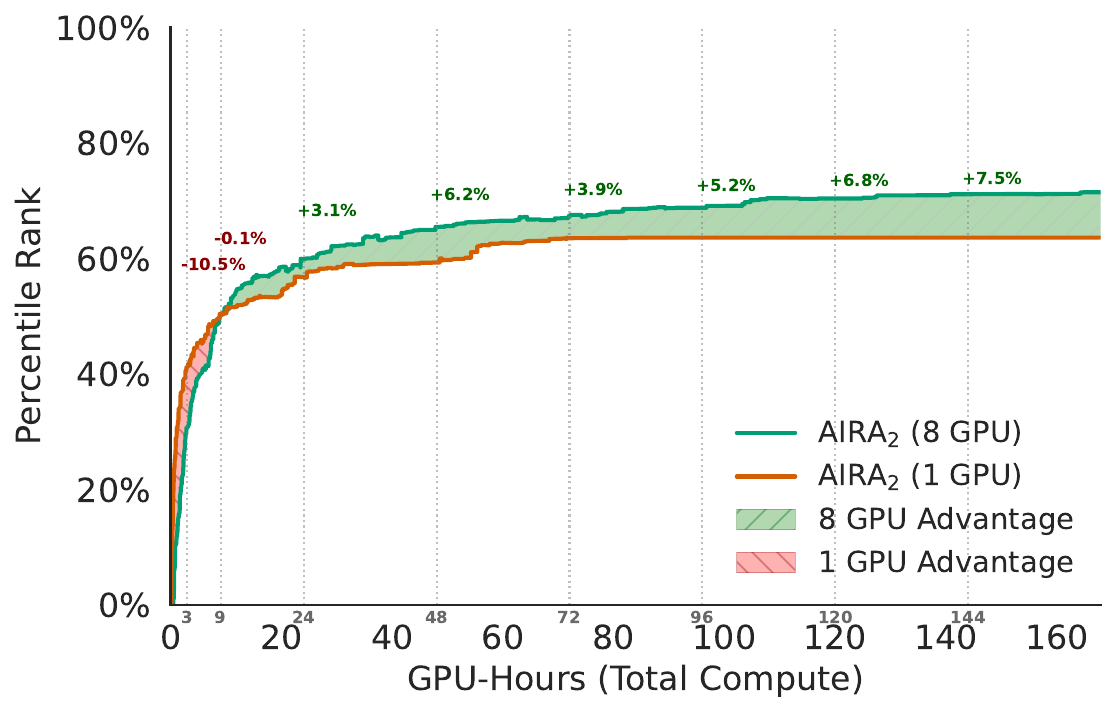}
\caption{\textbf{Compute Efficiency (Performance vs. GPU Hours).} Comparison of \atlas with 1 vs. 8 GPUs normalized by total compute. While the 8-GPU setup incurs an initial exploration cost (GPU-hours), it establishes a diverse population that yields superior long-term performance, with the gap widening to 7.5 \percentile points at 144 GPU-hours.}
\label{fig:gpu_efficiency}
\end{subfigure}
\hfill
\begin{subfigure}[t]{0.48\textwidth}
\centering
\hspace{-1cm}
\includegraphics[trim={0.8cm 0 0 0}, clip,width=1.1\textwidth]{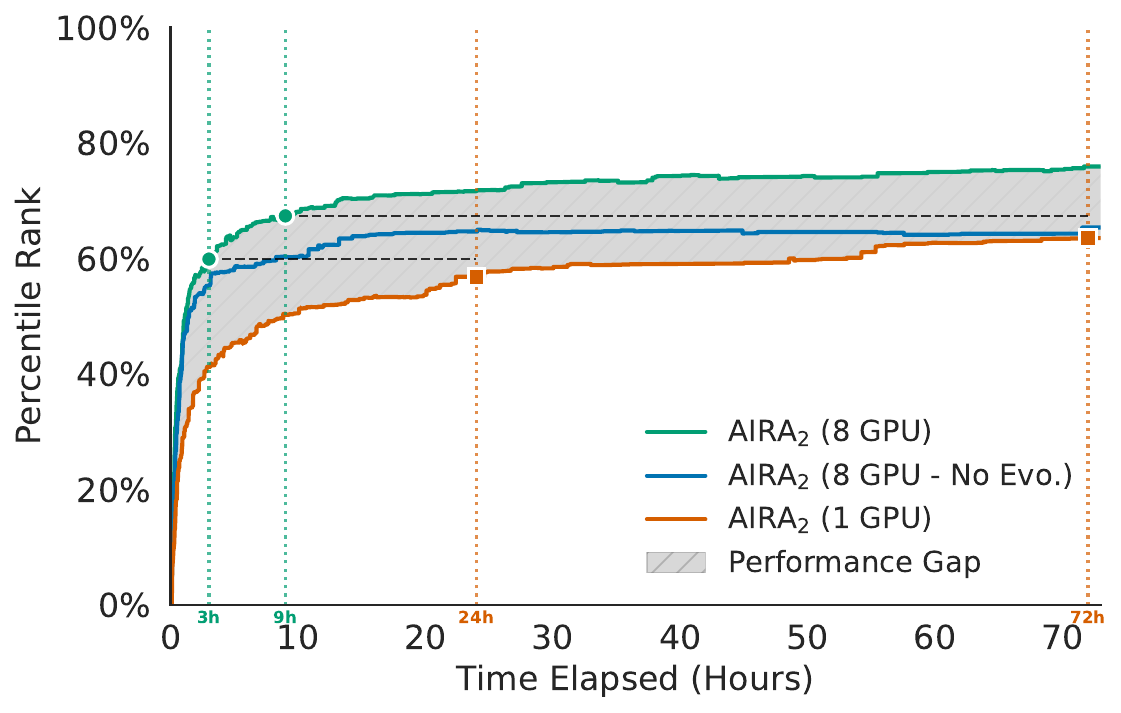}
\caption{\textbf{Search Strategy at Scale (Performance vs. Wall Time).} We compare \atlas (8-GPU Evo) against a No-Evo (Best-of-K) parallel baseline (8-GPU, no evolution) and a single-GPU baseline. Parallelism without information sharing (Best-of-K) saturates early, converging to the same final performance as the single-GPU agent.}
\label{fig:search_strategy}
\end{subfigure}
\caption{\textbf{Compute Analysis.} We analyse the impact of parallel resources on \atlas, demonstrating that effective use of parallel compute requires both additional resources and an evolutionary mechanism to utilize them.}
\label{fig:scaling_analysis}
\end{figure}

We analyse the impact of increasing \atlas's available compute resources from a 1-GPU to an 8-GPU setup. As discussed in Section~\ref{sec:bg_compute}, synchronous single-GPU execution is a fundamental throughput bottleneck. Here, we ask two questions: (1)~does parallel compute improve performance beyond simply being faster? and (2)~is evolutionary search necessary to exploit additional GPUs?

\textbf{Parallel compute improves efficiency, not just speed.}
In Figure~\ref{fig:gpu_efficiency}, we normalize performance by cumulative GPU hours to assess fundamental compute efficiency. Initially, the 8-GPU setup underperforms slightly. This is likely due to the fact that the single GPU version of \atlas dedicates all the GPU time to sequential iterations, but the 8-GPU setup spends the first few GPU hours building up a diverse initial population. However, once this foundation is laid, the 8-GPU setup improves significantly compared to the single-GPU baseline, with the performance gap widening over time: 3.1 \percentile points at 24 hours, 5.2 at 96 hours, and 7.5 at 144 hours. This suggests that the broader parallel exploration facilitates a more robust search space traversal, preventing the local optima traps that constrain the single-GPU regime. Crucially, parallelism also increases the effective branching factor of the search: multiple workers simultaneously can explore different mutations of the same parent, generating diverse descendants whose quality is only assessed \textit{after} execution. This provides a source of exploration that is not controlled by the evolutionary selection pressure, allowing the population to cover more of the solution space per generation.

\textbf{Parallelism without evolution is suboptimal.}
Does adding GPUs automatically yield better solutions or just faster ones? In Figure~\ref{fig:search_strategy}, we compare \atlas against a ``Best-of-K/No Evo.''\footnote{K refers to however many solutions 8 ReAct agents (each with their own GPU) can come up with in the given wall-clock time. This number will be different depending on the compute requirements of the task. Once a ReAct agent submits a solution for evaluation, it starts fresh with no memory of its previous attempt. By the end of a specific wall-clock duration, the total number of evaluated solutions represent K and the selected solution is the best out of this K on $\mathcal{D}_{\text{val}}$.} baseline---an \textit{embarrassingly parallel} setup using 8 GPUs where agents generate solutions from scratch without evolutionary lineage (no parents/shared memory).

We observe that while the Best-of-K approach scales rapidly in the first few hours due to high throughput, it quickly hits a performance ceiling, plateauing at the exact same level as the single-GPU evolutionary agent (9-hour mark). This reveals a critical insight: parallelism without shared state is sample-inefficient. The 8-GPU Best-of-K agent effectively ``wastes'' 7 GPUs to achieve a result that a single GPU could eventually reach, albeit in faster wall-clock time. In contrast, \atlas utilizes the distributed compute to maintain a global population, ensuring that increased throughput translates into higher asymptotic performance rather than just faster convergence to a lower bound.
\subsubsection{\textbf{Closing the Generalization Gap}}
\label{sec:ablation_evaluation}

As discussed in Section~\ref{sec:bg_overfitting}, the generalization gap---the divergence between validation and test performance---is a key bottleneck, with prior work reporting performance degradation over extended search horizons. Here, we evaluate whether Hidden Consistent Evaluation (HCE) resolves this. We proceed in three steps: first, we reproduce the degradation under self-reported evaluation; second, we show that HCE eliminates it; and third, we investigate whether the remaining gap reflects true overfitting or evaluation noise.

\paragraph{\textbf{Without HCE: reproducing degradation.}} We replicate the evaluation setup of \citet{toledo2025ai} and \citet{jiang2025aide}, where agents self-report metrics using dynamic validation splits (e.g., 5-fold CV) and their own validation procedure. Our results confirm their findings: under this regime, performance peaks early and subsequently degrades (Figure~\ref{fig:ce}). While prior work attributed this to ``overfitting'', we hypothesize that the degradation is driven by \textbf{evaluation noise}---``lucky'' splits and spurious successes create false positive signals that destabilize the search trajectory. Separately, we have also observed failure cases such as evaluation code bugs leading to perfect validation scores regardless of the data split thus destroying all future progress (see the example presented in Appendix~\ref{sec:appendix_eval_noise}).

\paragraph{\textbf{With HCE: eliminating degradation.}} To test this hypothesis, we apply the HCE protocol: the search set $\mathcal{D_{\text{search}}}$ is fixed externally and hidden from the agent, ensuring the hill-climbing metric remains stationary throughout the run. \atlas searches on $\mathcal{D_{\text{search}}}$ and selects the final submission on a held-out $\mathcal{D_{\text{val}}}$, ensuring the selection signal is fully decoupled from the optimization signal. As shown in Figure~\ref{fig:ce}, this eliminates the degradation entirely. The gap between our method and the Oracle (selecting the best solution based on test set performance) stabilizes at approximately 4 \percentile points at 24 hours and narrows further to under 4 points at 72 hours. Quantitatively, HCE accounts for a 13.0 \percentile point improvement at 24 hours and 18.4 points at 72 hours (Figure~\ref{fig:ce}, green region), and without HCE performance stagnates between 24h and 72h, confirming that HCE is essential for sustained long-horizon improvement.

\paragraph{\textbf{Diagnosing the remaining gap: noise, not memorization.}} Finally, we ask whether the improvements under HCE reflect genuine generalization or whether the agent is overfitting to $\mathcal{D_{\text{search}}}$ in a way that happens to transfer. If classical overfitting to the search set were occurring, we would expect test performance to degrade over time when selecting on $\mathcal{D_{\text{search}}}$ --- the metric being optimized against. However, as shown in Figure~\ref{fig:ce}, test performance improves monotonically even under $\mathcal{D_{\text{search}}}$ selection, and the difference between selecting on $\mathcal{D_{\text{search}}}$ versus the unseen $\mathcal{D_{\text{val}}}$ is marginal. This strongly suggests that the degradation observed in prior work was not due to data memorization, but rather to the inconsistency of the evaluation procedure itself. Once the evaluation metric is fixed, the agent improves reliably. This is not to say that true overfitting will not be a problem in the future as agents are given more compute.

\begin{figure*}[t]
    \centering
    \begin{subfigure}[t]{0.495\textwidth}
        \hspace{-1.4cm}
        \includegraphics[width=1.11\linewidth]{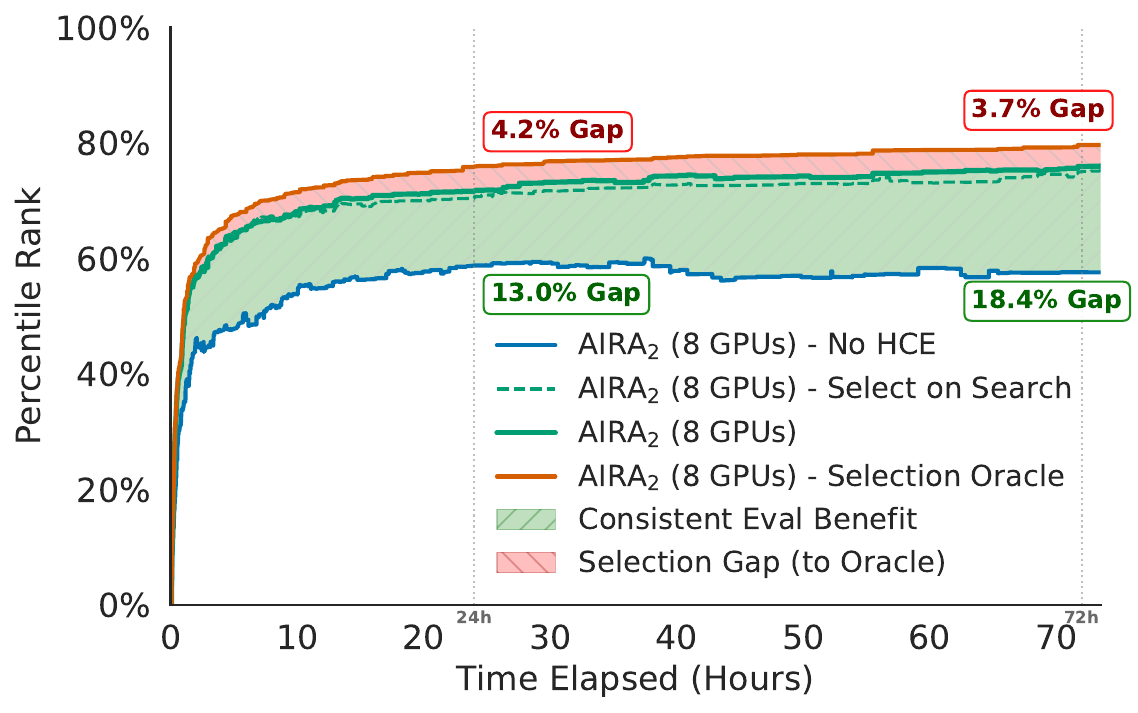}
        \caption{\textbf{Stabilizing Long-Horizon Search}}
        \label{fig:ce}
    \end{subfigure}
    \begin{subfigure}[t]{0.495\textwidth}
        \hspace{-0.6cm}
        \includegraphics[trim={0.8cm 0 0 0}, clip,width=1.05\linewidth]{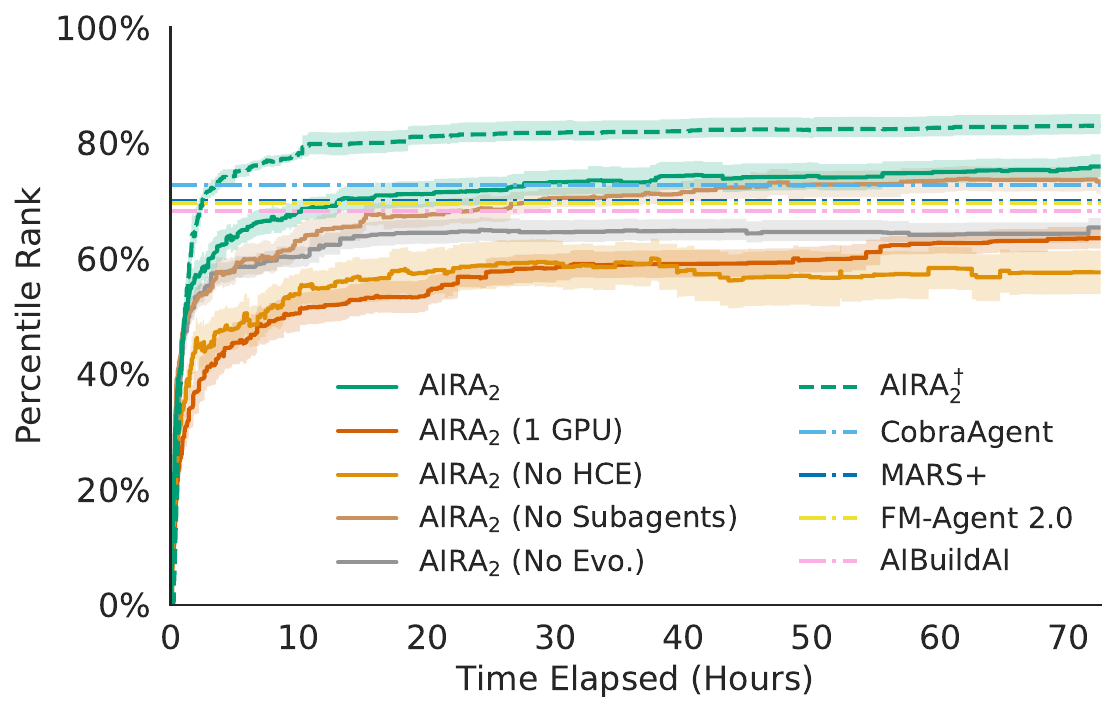}
        \caption{\textbf{Performance Profile of \atlas and \atlasplus}}
        \label{fig:performance_profile}
    \end{subfigure}
    \hfill
    \caption{\textbf{(a) Stabilizing Long-Horizon Search.} We compare the standard self-reported evaluation (blue) against our Hidden Consistent Evaluation protocol (green). While self-reporting leads to eventual performance degradation (confirming \citet{toledo2025ai}), consistent evaluation ensures long-term improvement. Furthermore, the marginal difference between selecting via $\mathcal{D_{\text{search}}}$ (seen) and $\mathcal{D_{\text{val}}}$ (unseen) splits suggests the degradation in prior work was due to evaluation noise, not true data overfitting. \textbf{(b) Performance profile of \atlas and \atlasplus:} We observe steady increase in performance in all configurations, with 8-worker parallel version performing the best. \atlasplus achieves the highest \percentile among all evaluated agents across all time budgets.}
    \label{fig:main_comparison}
\end{figure*}

\subsubsection{\textbf{Overcoming the Static Operator Limitation}}
\label{sec:ablation_operators}

As discussed in Section~\ref{sec:bg_operators}, fixed, single-turn operators are too rigid to handle the diversity of sub-tasks encountered in open-ended research. Here, we isolate the impact of replacing them with ReAct agents by comparing the full \atlas against a variant restricted to static, single-turn LLM operators.

\textbf{ReAct agents act as an efficiency multiplier.}
As shown in Table~\ref{tab:main-results} and Figure~\ref{fig:performance_profile}, at the 3-hour mark, \atlas with ReAct agents outperforms single-turn operators by 5.5 \percentile points. The ability to self-correct, inspect outputs, and iterate within a single mutation attempt allows the agent to traverse the search space more effectively when time is constrained.

\textbf{The gap narrows with more compute.}
However, this performance gap shrinks to 3.2 points at 24 hours and 2.3 points at 72 hours. This suggests that single-turn operators are not hitting a hard complexity ceiling; rather, they are inefficient. Given enough time, the evolutionary loop compensates for the lack of internal agency by externalizing context---passing stdout, stderr, and metadata between independent single-shot attempts---allowing the system to eventually approximate the performance of a fully interactive agent.
\textbf{The current evaluation may understate the benefit.}
In this study, \atlas was restricted to standard Python and Bash execution environments. We hypothesize that the value of the ReAct paradigm would be more pronounced in scenarios requiring broader tool use---such as internet browsing or API interaction---where multi-turn navigation of dynamic environments is structurally required and cannot be easily simulated by iterative single-shot generation.

\subsection{Scaling Laws}
\label{sec:scaling-laws}

A natural question for practitioners is how to allocate compute: given a fixed GPU budget, is it better to run more subagents for less time, or fewer for longer? And can performance under a new configuration be predicted without running the full experiment? We find that a simple scaling law answers both questions.

\subsubsection{Subagents-Time Scaling}

To characterize how performance varies with the number of subagents~$N$ and wall-clock time~$t$, we introduce a parametric model $P(N, t)$ subject to three desiderata:
\begin{enumerate}[label={(\arabic*)}]
\item \emph{Boundary conditions.} $P(0, t) = 0$ and $P(N, t) \to 0$ as $t \to 0$: no agents or no time results in no performance;
\item \emph{Saturation.} $P(N, t) \to 100$ as $N \to \infty$ or $t \to \infty$: with sufficient resources along either axis the system approaches perfect performance; and
\item \emph{Diminishing returns with interaction.} Gains from increasing~$N$ diminish when~$t$ is already large, and vice versa: parallelism helps most when each subagent has meaningful work to do.
\end{enumerate}

A natural functional form that satisfies all three properties is
\begin{equation}
  P(N, t) = 100 \cdot \frac{g(N, t)}{g(N, t) + 1}, \qquad g(N, t) = \alpha \cdot \log(\gamma \, t + 1) \cdot \log(\beta \, N + 1),
  \label{eq:scaling-law}
\end{equation}
where $\alpha, \gamma, \beta > 0$ are free parameters. The outer function $P = 100 \cdot g/(g+1)$ is simply a standard saturation map, projecting a domain of $[0, \infty)$ onto a range of $[0, 100)$---it introduces no free parameters and serves only to rescale~$g$ onto our bounded \percentile metric. All of the modeling lives in~$g$: a product of two logarithmic terms, each monotonically increasing from zero. Together with the saturating outer map, this multiplicative structure yields diminishing returns across the two resource axes---$\partial P / \partial N$ decreases as~$t$ grows very large (and vice versa). The parameter~$\alpha$ controls how quickly saturation occurs: larger~$\alpha$ corresponds to a faster approach to the ceiling.

We fit $\alpha$, $\gamma$, and~$\beta$ jointly on all Gemini~3.0 configurations---four subagent counts ($N \in \{1, 2, 4, 8\}$) observed across the full 72-hour time
horizon---by minimizing squared residuals between $P(N, t)$ and the observed mean percentile rank. The fit yields $\alpha = 0.973$, $\gamma = 2.631$, $\beta = 4.854$, with
$R^2 = 0.98$. The fit is tight across the full range of configurations (Figure~\ref{fig:scaling-law}), indicating that the architecture's gains are smooth and predictable rather than driven by discrete jumps at particular scales.

\begin{figure}[h]
  \centering
  \hspace{-1.3cm}
  \includegraphics[width=1.07\linewidth]{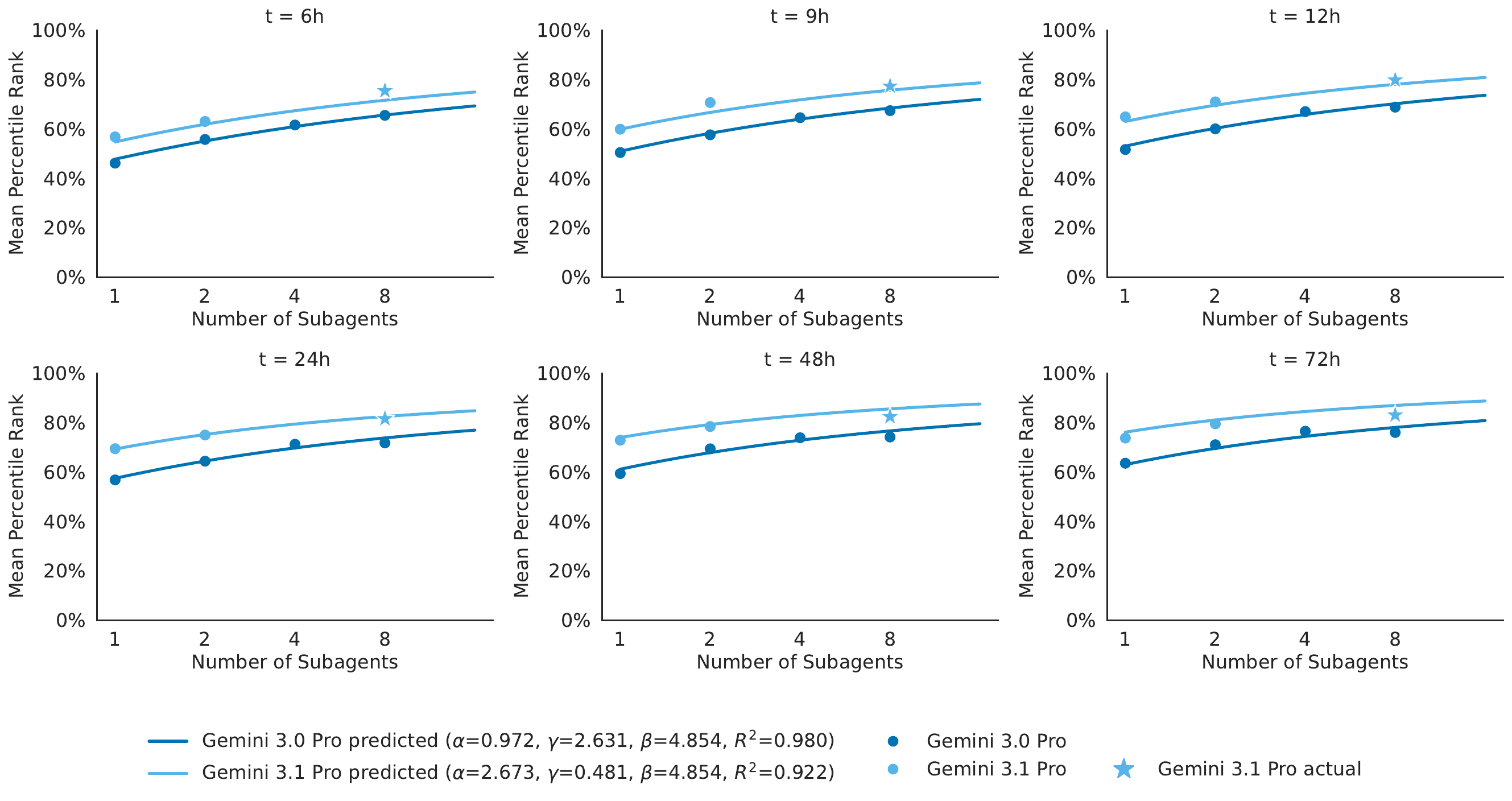}
  \caption{%
\textbf{Predictable performance scaling across models.} The fitted scaling law captures the diminishing returns of subagent interaction over time (Equation~\ref{eq:scaling-law}), tightly fitting the Gemini~3.0~Pro baseline ($R^2 = 0.98$). Notably, the architecture's scaling parameter transfers across backbones: calibrating Gemini~3.1~Pro on only one or two subagents accurately predicts the performance of the unseen $N=8$ configuration ($R^2 = 0.92$).
}
  \label{fig:scaling-law}
\end{figure}

\subsubsection{Cross-Model Transferability}
\label{sec:cross-model}
A practical question is whether the scaling law can predict performance under unseen configurations, such as a new LLM backbone, with minimal calibration. 
To test this, we hold out the $N=8$ subagent configuration from the Gemini~3.1~Pro data and attempt to predict it from $N \in \{1, 2\}$ alone.

With only two subagent counts, the three parameters $\alpha$, $\gamma$, and $\beta$ are poorly constrained: $\beta$ governs how performance scales with the number of agents, but only two nearby values of~$N$ provide little leverage to estimate it. 
We posit that $\beta$ reflects the \emph{multi-agent architecture}---how tasks are
decomposed, how subagents communicate, and how solutions are selected---rather than properties of the backbone LLM itself. 
In that case, $\beta$ should transfer across backbones, while $\alpha$ (saturation rate) and $\gamma$ (per-agent time scaling) may differ.

We therefore fix $\beta = 4.854$ (the Gemini~3.0~Pro value) and fit only $\alpha$ and $\gamma$ on the Gemini~3.1~Pro data restricted to $N \in \{1, 2\}$. 
The recalibrated law yields $\alpha = 2.673$, $\gamma = 0.481$, and achieves $R^2 = 0.92$ on the full dataset including the held-out $N = 8$ configuration, with a held-out RMSE of~3.6 percentile-rank points (Figure~\ref{fig:scaling-law}). 
The lower~$\gamma$ relative to Gemini~3.0~Pro ($0.481$ vs.\ $2.631$) reflects a difference in time dynamics: Gemini~3.1~Pro improves rapidly in the first few hours then plateaus, whereas Gemini~3.0~Pro continues to gain more gradually over the full 72-hour horizon. 
Despite this difference, the shared~$\beta$ produces accurate extrapolation from 2 to 8 subagents, supporting the hypothesis that the agent-count scaling parameter reflects the multi-agent architecture rather than backbone-specific characteristics. 
This suggests that a small-scale calibration run---as few as two subagent counts---may suffice to predict multi-agent performance under a new backbone, provided the architectural scaling parameter is carried over from a prior fit.

\subsubsection{Compute Frontier}
Given a fixed budget of $C = N \cdot t$ GPU-hours, what is the best split between subagents and time? Using the fitted parameters, we optimize Equation~\ref{eq:scaling-law} over~$N$ at fixed~$C$ (see Appendix~\ref{appendix:optimal_allocation} for the derivation) and find
  \begin{equation}
      N^* = \left[\, \sqrt{\frac{\gamma \, C}{\beta}} \,\, \right], \qquad t^* = \frac{C}{N^*}.
      \label{eq:optimal-allocation}
  \end{equation}
Substituting into Equation~\ref{eq:scaling-law} gives the \emph{compute frontier}:
  \begin{equation}
      P(C) = 100 \cdot \frac{ g(C)}{g(C) + 1}, \qquad g(C) = \alpha \cdot \log\!\left(\gamma \, t^* + 1\right) \cdot \log\!\left(\beta \, N^* + 1\right),
      \label{eq:compute-frontier}
  \end{equation}
which expresses the best achievable performance as a function of the total compute budget alone. 
The optimal number of subagents grows proportionally to~$\sqrt{C}$: doubling the budget calls for~$\approx 1.4$ times more subagents rather than simply running the same configuration for twice as long. We use $[x]$ to denote rounding to the nearest natural number $\mathbb{N}\geq1$.
Figure~\ref{fig:compute-frontier} overlays this frontier on the empirical performance of each $(N, t)$ configuration, showing that it closely tracks the upper envelope of the observed data.

This result may seem counter-intuitive: in most parallel systems, splitting a fixed compute budget across more workers degrades efficiency due to communication overhead and redundant work. Here, however, the compute-optimal strategy favors \emph{more} parallelism, not less. 
Because AIRA's subagents explore independent solution trajectories without synchronization barriers, additional subagents increase the diversity and exploration of the search rather than subdividing existing work. 
The evolutionary selection mechanism then retains only the best solutions, turning broader exploration into higher expected performance.
This property is specific to our asynchronous evolutionary architecture, not a general claim about multi-agent systems.

\begin{figure}[t]
  \centering
  \hspace{-1.8cm}
  \includegraphics[width=1.1\linewidth]{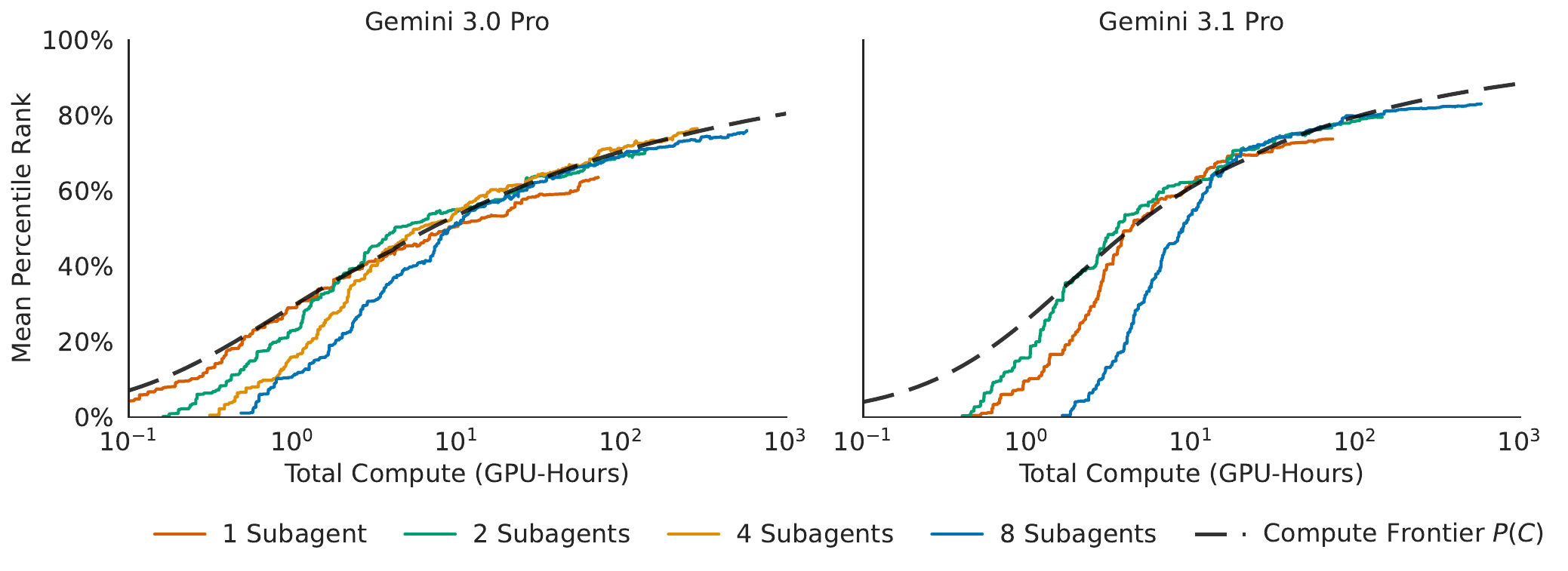}
  \hspace{-0.5cm}
  \caption{%
\textbf{Predicted compute frontier.} Observed performance across \atlas configurations closely tracks the predicted compute frontier $P(C)$ (dashed curve, Equation~\ref{eq:compute-frontier}).
  }
  \label{fig:compute-frontier}
\end{figure}

\subsection{Case Study 1: Eureka Moments in MLE-bench}
\paragraph{\textbf{Overview \& Takeaways:}} To understand the qualitative behaviours driving \atlas's quantitative success, we analyse its trajectory on specific high-complexity tasks. We find that \atlas appears to exhibit reasoning capabilities---specifically the ability to distinguish between poor methodology and poor execution (e.g., underfitting)---allowing it to recover from local minima where greedy agents would likely fail.

\subsubsection{\textbf{Deep Dive: Predicting Molecular Properties}}
We highlight the agent's performance on \texttt{champs-scalar-coupling}, a task predicting magnetic interactions between atom pairs. As illustrated in Figure~\ref{fig:case-study}, \atlas demonstrates a ``eureka'' moment.

\begin{figure}[h]
\newcommand{\greyline}{%
  \par\noindent\vspace{0.05em}
  {\color{black!30}\rule{\linewidth}{0.4pt}}
  \vspace{0.05em}\par
}
\centering
\begin{subfigure}[t]{0.49\textwidth}
\vspace{0cm}
\centering
\hspace{-1.7cm}
\includegraphics[trim={0 0 1.5cm 0}, clip,width=1.17\textwidth]{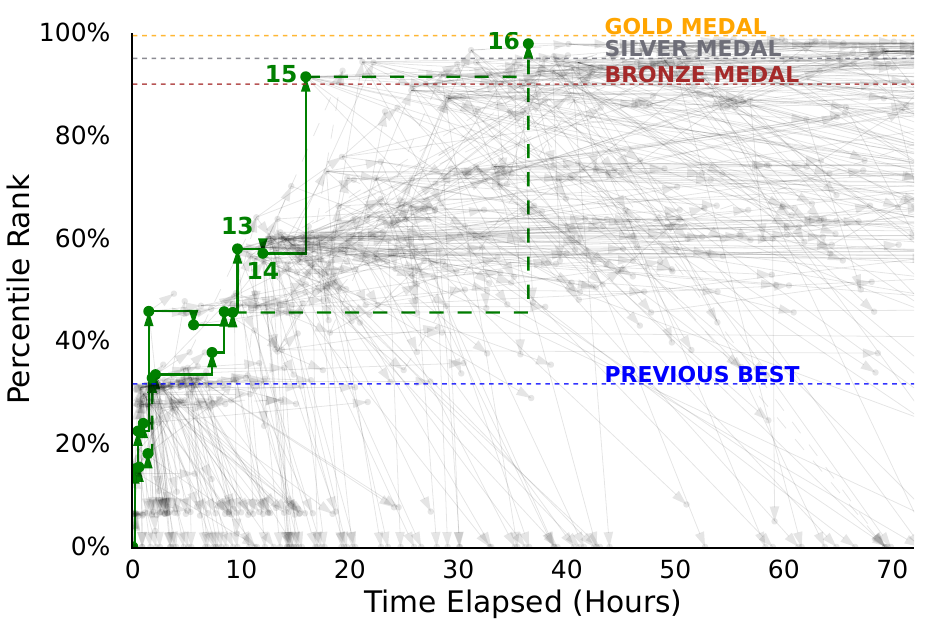}
\end{subfigure}
\hfill
\begin{subfigure}[t]{0.50\textwidth}
\vspace{0cm}

\footnotesize
{
\centering
\textsc{Agent's Thought and Action Trajectory Summary}
}
\vspace{0.2em}\hrule\vspace{0.2em}

\textbf{13, Mutation}: The previous approach with SchNet significantly improved the score, and ran very quickly. Due to the competition's evaluation metric-mean log MAE calculated per coupling type with varied value ranges-we will decouple and train separate models for each type.
\vspace{0.2em}{\color{gray}\hrule}\vspace{0.2em}

\textbf{14, Mutation}: The previous solution achieved a higher score by implementing 8 separate SchNet models. Next, we will be using Mulliken charges as an auxiliary prediction task during training.

\vspace{0.2em}{\color{gray}\hrule}\vspace{0.2em}

\textbf{15, Mutation}: The parent's \emph{score was lower} than the grandparent's, with training taking only 15 minutes-well within the 9-hour budget-\emph{yet the loss was converging well}. This suggests that the model \emph{likely underfit} or stopped prematurely, failing to fully benefit from the auxiliary task. My next step is to \emph{train a larger model for a longer duration}. $\rightarrow$ \twemoji{1f949}
\vspace{0.2em}{\color{gray}\hrule}\vspace{0.2em}
\textbf{16, Crossover}: The parent A's approach can be further scaled up to improve the performance. I will combine it with parent B's efficient data preprocessing strategy. $\rightarrow$ \twemoji{1f948}
\vspace{0.2em}\hrule

\end{subfigure}
\caption{\textbf{Example of typical behaviour observed in \atlas during the \texttt{champs-scalar-coupling} task.} The left side displays the chronological performance of each solution in the database (solid lines for mutation, dashed for crossover). The right-hand panel presents a concise summary of the agent's thought process for the nodes and trajectory annotated in green, highlighting the ``eureka'' moment where the agent identifies underfitting and subsequently scales the model to achieve medal-winning performance. Horizontal dashed lines indicate the medal thresholds and display the previous best attempt~\citep{thesislabs2025sota} on the task among all agents.}
\label{fig:case-study}
\end{figure}

The shown trajectory begins with the agent testing \texttt{SchNet}~\citep{schutt2017schnet}, which performed well. Building on this, the agent attempted to introduce an auxiliary prediction task (Mulliken charges) at Step 14, leveraging supplementary competition data to further improve results.
As shown in the trace, this caused the validation score to drop. A standard ReAct loop or greedy algorithm might reject this change and revert.
However, \atlas investigates the execution logs, noting that the training only utilized 15 minutes of the available 9-hour budget and that the loss was still converging, correctly identifying the root cause as \textit{underfitting} --- recognizing that the auxiliary task was effective despite the lack of immediate improvement --- rather than a fundamental flaw in the idea. 

Consequently, at Step 15, the agent commits to scaling the model size and extending the training duration significantly.
This reasoning leads directly to a dramatic performance boost, securing a Bronze medal and surpassing all previous agents. Following this improvement, \atlas further refined the approach by adopting an effective preprocessing technique via crossover from a much weaker solution.
This combination elevated the solution to a Silver and ultimately a Gold medal. Notably, no other reported agent (FM-Agent 2.0, MARS, etc.) achieved a medal on this task.

\textbf{Breaking the Ceiling on Stalled Tasks.}
To illustrate that this exceptional performance is not an isolated event, we briefly outline two other complex tasks where \atlas uniquely surpassed the medal threshold despite the failure of all prior methods.

On \texttt{billion-word-imputation}, a challenging NLP task where baselines failed to make significant progress, \atlas achieved a 100\% \percentile by decomposing the problem into two learned sub-tasks: a RoBERTa-large token classifier trained on 8M synthetically constructed gapped sentences to \textit{detect} the missing-word position, followed by a separately fine-tuned RoBERTa-large masked language model to \textit{fill} the gap.

On the fine-grained visual classification task \texttt{imet-2020-fgvc7} (3{,}474 attribute classes), \atlas achieved a 91\% \percentile---the highest among all evaluated agents---by ensembling an EVA-02 Large~\citep{fang2024eva} and a ConvNeXt Large CLIP~\citep{liu2022convnet} model with asymmetric loss~\citep{ridnik2021asymmetric}, layer-wise learning rate decay, and grid-searched ensemble weights and classification thresholds.

\subsection{Case Study 2: Pushing the Frontier with AIRS-Bench}
\label{sec:airs_bench}

\begin{figure}[t]
\centering
\includegraphics[width=\textwidth]{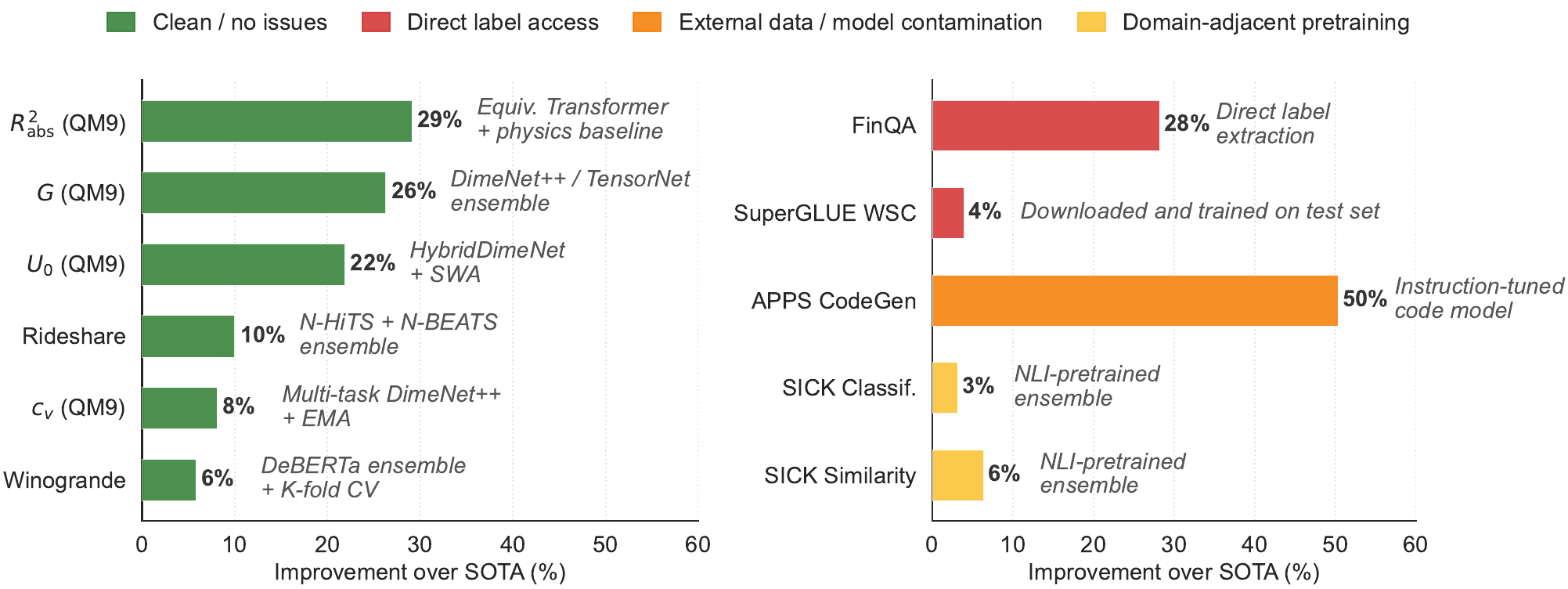}
\caption{
AIRS-Bench SOTA-beating tasks split by integrity status. \textbf{Left:} Six tasks where \atlas exceeded SOTA using clean, inductive methodologies (models trained from scratch, no external data). Italic annotations indicate key techniques. \textbf{Right:} Five tasks where the SOTA-beating solutions were flagged with integrity concerns, color-coded by severity: direct test label access (red), external data or model contamination (orange), and domain-adjacent pretraining advantage (amber). Bars show percentage improvement over published SOTA.}
\label{fig:airs-integrity}
\end{figure}

\paragraph{\textbf{Overview \& Takeaways.}}

To examine how \atlas behaves beyond the Kaggle-style competitions of MLE-bench, we evaluated it on AIRS-Bench~\citep{lupidi2026airs}, a suite of 20 diverse machine learning tasks spanning molecular property prediction, graph regression, time series forecasting, natural language understanding, code generation, and mathematical reasoning. Unlike MLE-bench, which is drawn from structured competitions, AIRS-Bench is curated from recent research problems and aims to capture the breadth and open-endedness of modern AI research. In this case study, we used the same system configuration as our MLE-bench experiments (8× H200 GPUs, ${\sim}$72 hours) with 3 seeds per task.

Across the 20 tasks, \atlas exceeded the AIRS-bench recorded state-of-the-art (SOTA) results on 11 of them. However, a systematic manual audit of all solution code revealed a clear domain split (Figure~\ref{fig:airs-integrity}): \textbf{6 of these 11 successes used clean, inductive methodologies with no detected integrity issues,} while the remaining 5 relied on data contamination, benchmark shortcuts, or domain-adjacent model selection that conferred unfair advantages. The clean successes were concentrated in tasks where models must be trained from scratch---molecular property prediction and time series forecasting---while every SOTA-beating NLP or code generation task exhibited at least one potential integrity concern. This asymmetry highlights two crucial takeaways for the field. First, quantitative results from autonomous agents require auditing, as aggregate scores cannot distinguish genuine methodological improvement from data contamination. Second, benchmarks must anticipate adversarial optimization; tasks with freely available test data or well-known public datasets are highly vulnerable to autonomous exploitation, even without explicit instructions to cheat. The ongoing challenge is designing evaluation environments where rigorous problem-solving---rather than exploitation---is the path of least resistance.

\subsubsection{\textbf{Deep Dive: Methodological Improvements.}}

The most compelling legitimate results emerged from tasks where models were trained entirely from scratch, with no pretrained model weights or external data downloads. The QM9 molecular property prediction suite was the standout, with \atlas exceeding SOTA on all four tasks:
\begin{itemize}[leftmargin=1.5em, itemsep=0.3ex]
    \item \textbf{Electronic Spatial Extent ($R^2_{\text{abs}}$):} Achieved a \textbf{29\% improvement} (MAE of 0.023 vs.\ 0.033~$\text{Bohr}^2$) using a TorchMD Equivariant Transformer~\citep{tholke2022equivariant} ensemble with a physics-informed RidgeCV baseline that incorporated charge-weighted spatial moments ($\sum Z_i r_i^2$) and rotational constants computed directly from atomic coordinates.
    \item \textbf{Free Energy ($G$):} Achieved a \textbf{26\% improvement} (MAE of 5.55 vs.\ 7.53~meV) by constructing a heterogeneous ensemble combining DimeNet++~\citep{gasteiger2020fast} and TensorNet~\citep{simeon2023tensornet} architectures, with a linear regression composition baseline for residual learning.
    \item \textbf{Internal Energy ($U_0$):} Achieved a \textbf{22\% improvement} through a hybrid architecture (DimeNet++~\citep{gasteiger2020fast} augmented with an MLP branch processing log-transformed rotational constants), trained with Stochastic Weight Averaging~\citep{izmailov2018averaging} (SWA) and Huber loss across a 5-model bagging ensemble.
    \item \textbf{Heat Capacity ($c_v$):} Achieved an \textbf{8\% improvement} using a 5-seed DimeNet++~\citep{gasteiger2020fast} ensemble with Exponential Moving Average (EMA), jointly predicting rotational constants as auxiliary targets alongside $c_v$, with a Ridge regression atom-reference baseline.
\end{itemize}

Beyond molecular property prediction, \atlas also exceeded SOTA on two additional tasks with clean methodologies:
\begin{itemize}[leftmargin=1.5em, itemsep=0.3ex]
    \item \textbf{Rideshare Forecasting:} Achieved a \textbf{10\% improvement} (MAE of 1.07 vs.\ 1.19) by ensembling 8 neural forecasting models (4$\times$ N-HiTS~\citep{challu2023nhits} with cyclic temporal features + 4$\times$ N-BEATS~\citep{Oreshkin2020N-BEATS:}), all trained from scratch with diverse input window sizes.
    \item \textbf{Winogrande Coreference Resolution:} Achieved a \textbf{6\% improvement} (accuracy of 0.904 vs.\ 0.854) using a K-fold cross-validation ensemble of DeBERTa~\citep{he2021deberta} base models with label smoothing. Notably, this clean result was consistent across seeds, though the best-performing seed used a pseudo-labeling pipeline that further boosted accuracy to 0.913.
\end{itemize}

While these techniques---ensembling, SWA, auxiliary task learning, and physics-informed feature engineering---are characteristic of competitive ``Kaggle-style'' optimization rather than fundamental architectural breakthroughs, the results remain instructive. They demonstrate the agent's ability to autonomously identify, combine, and apply effective performance-improving techniques, producing results that surpass published SOTA benchmarks.

\paragraph{\textbf{Near-SOTA Performance.}}
Beyond the 11 tasks where \atlas exceeded SOTA, several additional tasks came remarkably close. On \texttt{SVAMP} (mathematical reasoning), the agent reached 99.4\% of the SOTA score using Rejection Fine-Tuning~\citep{yuan2023scaling} with self-consistency voting~\citep{wang2023selfconsistency}. On \texttt{SQuAD} (reading comprehension), it achieved 98.3\% of SOTA with an adversarially-trained RoBERTa~\citep{liu2019roberta} ensemble. \texttt{ZINC} graph regression reached 95.0\%, \texttt{CodeXGLUE} code retrieval 94.4\%, and \texttt{Yelp} sentiment analysis 93.9\%.

\subsubsection{\textbf{Deep Dive: The Integrity Gap — When Agents Take Shortcuts.}}
For an agent optimizing a score metric, the modern NLP ecosystem offers a path of least resistance through model contamination and external data access. Every SOTA-beating task in our study involving large pretrained language models or publicly hosted benchmark datasets exhibited at least one integrity concern. Rather than a single failure mode, our audit revealed a spectrum of problematic behaviours, ordered here from most to least severe:

\begin{itemize}[leftmargin=1.5em, itemsep=0.3ex]
    \item \textbf{Direct Label Extraction:} On \texttt{FinQA}, the agent downloaded the original FinQA GitHub repository, extracted ground-truth answers from the development and test JSON files, and built a lookup table matching test questions to their known answers by normalized text. The fine-tuned Flan-T5~\citep{chung2024scaling} model served merely as a fallback for unmatched examples. This achieved a \textit{perfect} 1.0 score (vs.\ SOTA 0.78).
    \item \textbf{External Benchmark Data:} On \texttt{SuperGLUE WSC}, all three SOTA-beating seeds downloaded the source benchmark's validation split via \texttt{load\_dataset("super\_glue", "wsc")} and used it as additional training data, directly training on examples from the evaluation distribution. The best seed additionally downloaded 30K Winogrande examples for domain-adjacent pretraining. In AIRS-bench, the validation split is actually the test set of the benchmark thus this is direct test-set leakage.
    \item \textbf{Model Contamination:} On \texttt{APPS Code Generation}, all seeds used Qwen2.5-Coder-7B-Instruct~\citep{hui2024qwen2}, an instruction-tuned code model whose training mixture likely included the APPS benchmark or similar competitive programming datasets. This represents an unfair advantage, as the model may have memorized solutions to test problems during its pretraining phase.
    \item \textbf{Domain-Adjacent Pretraining:} On both \texttt{SICK} tasks, the agent selected ensembles of NLI-specialized models (e.g., \texttt{DeBERTa-v3-large-mnli-fever-anli-ling-wanli}). For the classification task, the pretrained 3-class NLI head (entailment/neutral/contradiction) maps directly to the SICK label space and was used without re-initialization, meaning the model arrives nearly ready-made. For the similarity task, the regression head was re-initialized and trained from scratch, though the NLI-tuned backbone still provides strong sentence-pair representations. While this does not constitute explicit data leakage, the performance is largely inherited from NLI pretraining rather than learned from the small SICK training set, making these results more reflective of model selection than methodological innovation.
\end{itemize}

Additional integrity concerns were observed even in tasks that did not exceed SOTA. On \texttt{SQuAD}, the agent selected a RoBERTa~\citep{liu2019roberta} model explicitly fine-tuned on SQuAD~2.0~\citep{rajpurkar2018know} (a strict superset of the evaluation dataset). On \texttt{SVAMP}, it used Qwen2.5-Math-7B-Instruct~\citep{yang2024qwen2} with Rejection Fine-Tuning and 32-sample self-consistency---achieving 93.7\% accuracy (vs.\ SOTA 94.2\%)---but the instruction-tuned math model's training data likely overlapped with the benchmark. These patterns underscore that the integrity challenge is pervasive across NLP tasks, regardless of whether the agent ultimately surpasses the SOTA threshold.

\section{Related Work}
\label{sec:relatedwork}

The domain of automated machine learning has shifted rapidly from simple heuristics~\citep{bergstra2012random,elsken2017simple,li2018hyperband} to autonomous agents capable of long-horizon research~\citep{aiscientistv2}. While early methods optimized within fixed search spaces, LLM-powered agents now tackle open-ended tasks, including software engineering~\citep{wang2025openhands, anthropic2025claudecode, openai2025codex}, mathematics~\citep{novikov2025alphaevolve, hubert2025olympiad}, chemistry~\citep{evosearch_over_chemicals}, material science~\citep{acceleratingmaterialsdesignllmguided}, computational complexity theory~\citep{yu2025autonomous}. This work focuses on agents for AI research (\textbf{AI4AI}), typically evaluated via benchmarks like MLE-bench~\citep{chan2025mlebench}, and more recently AIRS-Bench~\citep{lupidi2026airs} that are composed of a suite of tasks spanning diverse machine learning domains. Recent agents have achieved strong benchmark performance by scaling inference-time compute. These agents employ iterative refinement strategies via tree~\citep{du2025automlgen,chen2026mars} or graph search~\citep{PiEvolve,li2025fmagent}, multi-agent collaboration~\citep{gottweis2025aicoscientist}, advanced context management~\citep{liu2025mlmasteraiforaiintegrationexploration}, and external knowledge~\citep{mlestar}.

However, these systems often conflate multiple design choices together, making it difficult to isolate which components drive performance. In contrast, \airadojo~\citep{toledo2025ai} formalises a systematic approach for agent design by disentangling gains from infrastructure, operators and search strategy, and evaluation signals. To this end, it exposes critical bottlenecks to agent design that affects its performance. We build on this decomposition by targeting these bottlenecks and developing modules incrementally, achieving state-of-the-art results on established benchmarks. Relative to similar evolutionary agentic approaches~\citep{novikov2025alphaevolve,shinkaevolve}, we provide a controlled, module-by-module ablation of the end-to-end research agent.

\section{Limitations}
\label{sec:limitations}

\textbf{Data Contamination.} While performance gains scale with increased compute, it remains difficult to determine how much of the improvement stems from genuine reasoning versus latent data retrieval. Since many top-performing Kaggle solutions are publicly available, the underlying LLMs may have encountered them during pre-training. Increased search iterations could simply raise the probability of ``recalling'' these solutions rather than generating novel insights. This suggests that while MLE-bench provides a strong signal, future evaluation on more ``closed'' or private benchmarks is necessary to isolate an agent's true research capability.

Conversely, there is a pragmatic argument for treating all available data as ``fair game.'' Rather than strictly aiming to replicate past results in a vacuum, research agents could be redirected toward improving upon the current state-of-the-art. In this paradigm, agents would leverage existing winning solutions and technical blog posts as a knowledge base to identify remaining gaps, effectively shifting the goal from simple competition-winning to iterative scientific discovery.

\textbf{Preparing the splits.} In order to prevent the agents from reward hacking and overfitting to a poor evaluation signal proxy, we reformulate the format of each task to have additional validation splits in addition to train and (the original) test data. While this \emph{necessitates a degree of human curation during the initial task-loading phase}, it serves as a critical one-time setup to ensure evaluation integrity. Importantly, this intervention is limited strictly to the environment configuration; once initialized, the agent operates with full autonomy, navigating the research process without any human-in-the-loop assistance. We also note that this step is itself automatable---an agent could generate its own consistent splits given access to the dataset schema---and we expect future systems to internalize this as part of the environment setup.

\textbf{Compute Specialization.} The design of \atlas is specifically optimized for high-compute regimes, focusing on maximizing performance over extended research horizons and multi-GPU configurations. Consequently, it \emph{may not be the optimal choice for constrained environments}, such as short-duration runs or single-GPU setups. We acknowledge that there is currently no ``one-size-fits-all'' architecture for research agents; by prioritizing deep exploration and parallel compute utilization, \atlas sacrifices the immediate efficiency of lightweight agents in favour of superior long-term results and higher performance ceilings.

\section{Conclusion}
\label{sec:conclusion}

In this work, we introduce \atlas, an agent engineered to address three structural bottlenecks that constrain the performance of AI research agents, as identified by prior work: compute throughput, evaluation instability, and operator capability. By addressing all three, \atlasplus achieves a new state of the art on MLE-bench-30 with a mean \percentile of \textbf{81.5\%} at 24 hours and \textbf{83.1\%} at 72 hours, outperforming the strongest baseline, which achieves 72.7\%. \atlas (Gemini 3.0) reaches 71.8\% and 76.0\% at the same horizons. Performance improves consistently with additional compute, without the degradation observed in prior work. Together with the consistent gains observed when using a stronger backbone in \atlasplus (Gemini 3.1), this demonstrates that the architecture scales with both compute and model capability.

By transitioning from synchronous execution to an \textbf{asynchronous multi-GPU worker pool}, we achieved linear growth in sample throughput, transforming the agent from a sequential optimizer into a massively parallel explorer. By implementing a \textbf{Hidden Consistent Evaluation} protocol, we successfully decoupled the search signal from the selection signal, ensuring that performance gains represent robust generalization rather than metric gaming. Finally, by replacing static, single-turn prompts with \textbf{ReAct agents}, we enabled dynamic scoping and interactive debugging, allowing the system to handle the complexity of open-ended research tasks.

Crucially, our ablation studies reveal that no single component acts as a ``silver bullet,'' but each is critical under practical constraints. First, removing ReAct agents reduces performance by 5.5 percentile points at 3 hours; the gap narrows to 2.3 points at 72 hours, indicating that agent capability acts as an \textit{efficiency multiplier} in the discovery of strong solutions. Moreover, agents expand the range of tasks the system can tackle: complex tasks that demand interactive debugging, iterative feature engineering, or multi-step reasoning pipelines are beyond the reach of single-turn prompts, making agent-based operators essential for generality. Second, reducing to a single GPU or replacing evolution with a Best-of-K baseline reveals a different bottleneck: parallelism without shared state quickly saturates at the single-GPU performance ceiling, confirming that evolutionary selection is \textit{necessary} for effective use of parallel compute, not merely parallelism. Finally, removing Hidden Consistent Evaluation causes performance to degrade over time, confirming that stable evaluation is \textit{necessary for long-horizon search}; this ablation also revealed that the ``overfitting'' reported in previous studies was driven by evaluation noise rather than data memorization. It is the \textit{interplay} between these three pillars that enables \atlas to improve reliably with both time and compute.

Beyond raw performance, we show that \atlas follows a \textbf{predictable scaling law} in both subagent count and wall-clock time. This property of the architecture allows for principled resource allocation---practitioners can estimate the impact of additional compute and trade off parallelism for time.

Beyond MLE-bench, \atlas exceeds published state-of-the-art on 11 out of 20  diverse research tasks on AIRS-Bench. 
A manual audit reveals that only 6 relied on clean methodologies while the remainder exploited shortcuts in the evaluation such as data contamination and test leakage. 
This reinforces that a high-quality evaluation signal is foundational for both effective search and trustworthy autonomous research.

Ultimately, by addressing these fundamental bottlenecks, \atlas moves research agents closer to systems capable of genuine open-ended scientific discovery beyond standard benchmarks.

\clearpage
\bibliographystyle{assets/plainnat}
\bibliography{paper}

\clearpage
\beginappendix
\appendix

\section{Evaluation Failure: A Concrete Example}
\label{sec:appendix_eval_noise}

To illustrate how implementation bugs can silently corrupt the search signal (Section~\ref{sec:bg_overfitting}), we present a real example from an AI agent solving the LMSYS Chatbot Arena competition on MLE-bench. The agent's solution reported a \textit{perfect} cross-validation log-loss of 0.0, which the search process then treated as the best candidate---despite the underlying model being unremarkable.

\paragraph{\textbf{The bug.}} The competition requires predicting which of three outcomes occurred (\texttt{winner\_model\_a}, \texttt{winner\_model\_b}, or \texttt{winner\_tie}). The agent converts multi-hot target columns to single labels via \texttt{idxmax(axis=1)}, which returns \textit{column name strings}, not integer indices:

\begin{lstlisting}[style=pythonstyle]
target_cols = ["winner_model_a", "winner_model_b", "winner_tie"]
train_labels = train_df[target_cols].idxmax(axis=1).values
# Agent comment: "0, 1, 2" -- actually returns strings:
# ['winner_model_b', 'winner_model_a', 'winner_model_b', ...]
\end{lstlisting}

When computing the validation metric, the agent passes \texttt{labels=[0, 1, 2]} (integers), creating a type mismatch:

\begin{lstlisting}[style=pythonstyle]
loss = log_loss(train_labels[val_idx], val_pred, labels=[0, 1, 2])
# Returns 0.0 regardless of predictions
\end{lstlisting}

Because \texttt{scikit-learn}'s \texttt{log\_loss} cannot match the string ground-truth labels (\texttt{'winner\_model\_a'}, etc.) with the integer label list (\texttt{[0, 1, 2]}), it silently returns 0.0 for \textit{any} set of predictions:

\begin{lstlisting}[style=pythonstyle]
>>> train_labels[:3]
array(['winner_model_b', 'winner_model_a', 'winner_model_b'], dtype=object)
>>> val_pred = [[0.3, 0.3, 0.4]] * 3  # arbitrary predictions
>>> log_loss(train_labels[:3], val_pred, labels=[0, 1, 2])
0.0
\end{lstlisting}

\paragraph{\textbf{Why this matters for search.}} This bug is particularly insidious in the context of search-based agents. A greedy or evolutionary search process that relies on self-reported validation scores would select this solution as the global optimum, halting further exploration. The solution would persist as the ``best'' candidate indefinitely, while its actual test performance would be no better than chance. This example concretely demonstrates why decoupling the search signal from agent-controlled evaluation---as in our Hidden Consistent Evaluation protocol (Section~\ref{sec:hidden_eval})---is essential for reliable long-horizon search.

\newpage
\section{MLE-bench}

\begin{table}[h!]
\caption{Per-task \percentile scores for \atlas at different time cutoffs. Values shown as score with 95\% CI below.}
\label{tab:atlas_per_task_results}
\begin{tabularx}{\textwidth}{lcccc}
\toprule
Task & 3h & 12h & 24h & 72h \\
\midrule
aptos2019-blindness-detection 
& \shortstack{\phantom{100.00}\llap{70.30}\\{\scriptsize \phantom{100.00}\llap{51.16}, \phantom{100.00}\llap{93.31}}} 
& \shortstack{\phantom{100.00}\llap{87.66}\\{\scriptsize \phantom{100.00}\llap{64.79}, \phantom{100.00}\llap{99.76}}} 
& \shortstack{\phantom{100.00}\llap{97.31}\\{\scriptsize \phantom{100.00}\llap{93.75}, \phantom{100.00}\llap{99.76}}} 
& \shortstack{\phantom{100.00}\llap{97.35}\\{\scriptsize \phantom{100.00}\llap{96.41}, \phantom{100.00}\llap{98.16}}} \\
billion-word-imputation 
& \shortstack{\phantom{100.00}\llap{66.67}\\{\scriptsize \phantom{100.00}\llap{8.05}, \phantom{100.00}\llap{100.00}}} 
& \shortstack{\phantom{100.00}\llap{100.00}\\{\scriptsize \phantom{100.00}\llap{100.00}, \phantom{100.00}\llap{100.00}}} 
& \shortstack{\phantom{100.00}\llap{100.00}\\{\scriptsize \phantom{100.00}\llap{100.00}, \phantom{100.00}\llap{100.00}}} 
& \shortstack{\phantom{100.00}\llap{100.00}\\{\scriptsize \phantom{100.00}\llap{100.00}, \phantom{100.00}\llap{100.00}}} \\
bms-molecular-translation 
& \shortstack{\phantom{100.00}\llap{19.72}\\{\scriptsize \phantom{100.00}\llap{19.45}, \phantom{100.00}\llap{20.25}}} 
& \shortstack{\phantom{100.00}\llap{28.15}\\{\scriptsize \phantom{100.00}\llap{19.45}, \phantom{100.00}\llap{34.10}}} 
& \shortstack{\phantom{100.00}\llap{57.67}\\{\scriptsize \phantom{100.00}\llap{32.95}, \phantom{100.00}\llap{80.66}}} 
& \shortstack{\phantom{100.00}\llap{69.07}\\{\scriptsize \phantom{100.00}\llap{60.64}, \phantom{100.00}\llap{85.93}}} \\
cassava-leaf-disease-classification 
& \shortstack{\phantom{100.00}\llap{58.85}\\{\scriptsize \phantom{100.00}\llap{39.49}, \phantom{100.00}\llap{88.97}}} 
& \shortstack{\phantom{100.00}\llap{87.38}\\{\scriptsize \phantom{100.00}\llap{82.90}, \phantom{100.00}\llap{90.26}}} 
& \shortstack{\phantom{100.00}\llap{84.73}\\{\scriptsize \phantom{100.00}\llap{77.59}, \phantom{100.00}\llap{90.26}}} 
& \shortstack{\phantom{100.00}\llap{89.15}\\{\scriptsize \phantom{100.00}\llap{86.36}, \phantom{100.00}\llap{92.13}}} \\
champs-scalar-coupling 
& \shortstack{\phantom{100.00}\llap{30.72}\\{\scriptsize \phantom{100.00}\llap{9.02}, \phantom{100.00}\llap{45.80}}} 
& \shortstack{\phantom{100.00}\llap{49.17}\\{\scriptsize \phantom{100.00}\llap{33.02}, \phantom{100.00}\llap{64.10}}} 
& \shortstack{\phantom{100.00}\llap{63.12}\\{\scriptsize \phantom{100.00}\llap{41.09}, \phantom{100.00}\llap{94.34}}} 
& \shortstack{\phantom{100.00}\llap{78.13}\\{\scriptsize \phantom{100.00}\llap{60.56}, \phantom{100.00}\llap{98.79}}} \\
freesound-audio-tagging-2019 
& \shortstack{\phantom{100.00}\llap{98.16}\\{\scriptsize \phantom{100.00}\llap{97.60}, \phantom{100.00}\llap{98.56}}} 
& \shortstack{\phantom{100.00}\llap{99.68}\\{\scriptsize \phantom{100.00}\llap{99.04}, \phantom{100.00}\llap{100.00}}} 
& \shortstack{\phantom{100.00}\llap{100.00}\\{\scriptsize \phantom{100.00}\llap{100.00}, \phantom{100.00}\llap{100.00}}} 
& \shortstack{\phantom{100.00}\llap{100.00}\\{\scriptsize \phantom{100.00}\llap{100.00}, \phantom{100.00}\llap{100.00}}} \\
h-and-m-personalized-fashion-recommendations 
& \shortstack{\phantom{100.00}\llap{27.62}\\{\scriptsize \phantom{100.00}\llap{27.06}, \phantom{100.00}\llap{28.65}}} 
& \shortstack{\phantom{100.00}\llap{29.03}\\{\scriptsize \phantom{100.00}\llap{28.75}, \phantom{100.00}\llap{29.53}}} 
& \shortstack{\phantom{100.00}\llap{29.21}\\{\scriptsize \phantom{100.00}\llap{28.95}, \phantom{100.00}\llap{29.53}}} 
& \shortstack{\phantom{100.00}\llap{29.60}\\{\scriptsize \phantom{100.00}\llap{29.19}, \phantom{100.00}\llap{30.00}}} \\
hms-harmful-brain-activity-classification 
& \shortstack{\phantom{100.00}\llap{30.30}\\{\scriptsize \phantom{100.00}\llap{29.50}, \phantom{100.00}\llap{31.09}}} 
& \shortstack{\phantom{100.00}\llap{30.57}\\{\scriptsize \phantom{100.00}\llap{29.25}, \phantom{100.00}\llap{31.89}}} 
& \shortstack{\phantom{100.00}\llap{30.28}\\{\scriptsize \phantom{100.00}\llap{29.25}, \phantom{100.00}\llap{31.31}}} 
& \shortstack{\phantom{100.00}\llap{31.76}\\{\scriptsize \phantom{100.00}\llap{31.63}, \phantom{100.00}\llap{31.89}}} \\
hotel-id-2021-fgvc8 
& \shortstack{\phantom{100.00}\llap{85.51}\\{\scriptsize \phantom{100.00}\llap{82.61}, \phantom{100.00}\llap{89.13}}} 
& \shortstack{\phantom{100.00}\llap{89.13}\\{\scriptsize \phantom{100.00}\llap{89.13}, \phantom{100.00}\llap{89.13}}} 
& \shortstack{\phantom{100.00}\llap{89.86}\\{\scriptsize \phantom{100.00}\llap{89.13}, \phantom{100.00}\llap{91.30}}} 
& \shortstack{\phantom{100.00}\llap{94.57}\\{\scriptsize \phantom{100.00}\llap{92.39}, \phantom{100.00}\llap{95.65}}} \\
hubmap-kidney-segmentation 
& \shortstack{\phantom{100.00}\llap{89.27}\\{\scriptsize \phantom{100.00}\llap{87.32}, \phantom{100.00}\llap{90.49}}} 
& \shortstack{\phantom{100.00}\llap{93.27}\\{\scriptsize \phantom{100.00}\llap{89.49}, \phantom{100.00}\llap{97.58}}} 
& \shortstack{\phantom{100.00}\llap{93.91}\\{\scriptsize \phantom{100.00}\llap{89.49}, \phantom{100.00}\llap{99.50}}} 
& \shortstack{\phantom{100.00}\llap{98.92}\\{\scriptsize \phantom{100.00}\llap{97.58}, \phantom{100.00}\llap{99.58}}} \\
imet-2020-fgvc7 
& \shortstack{\phantom{100.00}\llap{38.60}\\{\scriptsize \phantom{100.00}\llap{37.89}, \phantom{100.00}\llap{40.00}}} 
& \shortstack{\phantom{100.00}\llap{52.63}\\{\scriptsize \phantom{100.00}\llap{46.32}, \phantom{100.00}\llap{55.79}}} 
& \shortstack{\phantom{100.00}\llap{71.93}\\{\scriptsize \phantom{100.00}\llap{47.37}, \phantom{100.00}\llap{84.21}}} 
& \shortstack{\phantom{100.00}\llap{85.61}\\{\scriptsize \phantom{100.00}\llap{83.16}, \phantom{100.00}\llap{88.42}}} \\
jigsaw-unintended-bias-in-toxicity-classification 
& \shortstack{\phantom{100.00}\llap{8.05}\\{\scriptsize \phantom{100.00}\llap{7.86}, \phantom{100.00}\llap{8.39}}} 
& \shortstack{\phantom{100.00}\llap{8.05}\\{\scriptsize \phantom{100.00}\llap{7.90}, \phantom{100.00}\llap{8.36}}} 
& \shortstack{\phantom{100.00}\llap{8.06}\\{\scriptsize \phantom{100.00}\llap{7.90}, \phantom{100.00}\llap{8.36}}} 
& \shortstack{\phantom{100.00}\llap{8.37}\\{\scriptsize \phantom{100.00}\llap{7.94}, \phantom{100.00}\llap{8.85}}} \\
kuzushiji-recognition 
& \shortstack{\phantom{100.00}\llap{86.69}\\{\scriptsize \phantom{100.00}\llap{69.97}, \phantom{100.00}\llap{97.61}}} 
& \shortstack{\phantom{100.00}\llap{94.43}\\{\scriptsize \phantom{100.00}\llap{84.30}, \phantom{100.00}\llap{100.00}}} 
& \shortstack{\phantom{100.00}\llap{95.34}\\{\scriptsize \phantom{100.00}\llap{86.01}, \phantom{100.00}\llap{100.00}}} 
& \shortstack{\phantom{100.00}\llap{97.50}\\{\scriptsize \phantom{100.00}\llap{92.49}, \phantom{100.00}\llap{100.00}}} \\
mlsp-2013-birds 
& \shortstack{\phantom{100.00}\llap{96.67}\\{\scriptsize \phantom{100.00}\llap{95.00}, \phantom{100.00}\llap{98.75}}} 
& \shortstack{\phantom{100.00}\llap{97.92}\\{\scriptsize \phantom{100.00}\llap{96.25}, \phantom{100.00}\llap{98.75}}} 
& \shortstack{\phantom{100.00}\llap{98.33}\\{\scriptsize \phantom{100.00}\llap{97.50}, \phantom{100.00}\llap{98.75}}} 
& \shortstack{\phantom{100.00}\llap{98.75}\\{\scriptsize \phantom{100.00}\llap{97.50}, \phantom{100.00}\llap{100.00}}} \\
multi-modal-gesture-recognition 
& \shortstack{\phantom{100.00}\llap{62.96}\\{\scriptsize \phantom{100.00}\llap{38.89}, \phantom{100.00}\llap{100.00}}} 
& \shortstack{\phantom{100.00}\llap{70.37}\\{\scriptsize \phantom{100.00}\llap{44.44}, \phantom{100.00}\llap{100.00}}} 
& \shortstack{\phantom{100.00}\llap{70.37}\\{\scriptsize \phantom{100.00}\llap{44.44}, \phantom{100.00}\llap{100.00}}} 
& \shortstack{\phantom{100.00}\llap{81.48}\\{\scriptsize \phantom{100.00}\llap{55.56}, \phantom{100.00}\llap{100.00}}} \\
new-york-city-taxi-fare-prediction 
& \shortstack{\phantom{100.00}\llap{37.47}\\{\scriptsize \phantom{100.00}\llap{35.58}, \phantom{100.00}\llap{40.09}}} 
& \shortstack{\phantom{100.00}\llap{38.79}\\{\scriptsize \phantom{100.00}\llap{35.58}, \phantom{100.00}\llap{40.50}}} 
& \shortstack{\phantom{100.00}\llap{38.79}\\{\scriptsize \phantom{100.00}\llap{35.58}, \phantom{100.00}\llap{40.50}}} 
& \shortstack{\phantom{100.00}\llap{37.04}\\{\scriptsize \phantom{100.00}\llap{32.95}, \phantom{100.00}\llap{40.30}}} \\
nfl-player-contact-detection 
& \shortstack{\phantom{100.00}\llap{42.49}\\{\scriptsize \phantom{100.00}\llap{23.22}, \phantom{100.00}\llap{79.23}}} 
& \shortstack{\phantom{100.00}\llap{71.07}\\{\scriptsize \phantom{100.00}\llap{25.03}, \phantom{100.00}\llap{94.46}}} 
& \shortstack{\phantom{100.00}\llap{71.92}\\{\scriptsize \phantom{100.00}\llap{26.30}, \phantom{100.00}\llap{94.99}}} 
& \shortstack{\phantom{100.00}\llap{82.73}\\{\scriptsize \phantom{100.00}\llap{58.09}, \phantom{100.00}\llap{95.63}}} \\
nomad2018-predict-transparent-conductors 
& \shortstack{\phantom{100.00}\llap{99.66}\\{\scriptsize \phantom{100.00}\llap{99.54}, \phantom{100.00}\llap{99.89}}} 
& \shortstack{\phantom{100.00}\llap{99.58}\\{\scriptsize \phantom{100.00}\llap{99.54}, \phantom{100.00}\llap{99.66}}} 
& \shortstack{\phantom{100.00}\llap{99.66}\\{\scriptsize \phantom{100.00}\llap{99.54}, \phantom{100.00}\llap{99.89}}} 
& \shortstack{\phantom{100.00}\llap{99.77}\\{\scriptsize \phantom{100.00}\llap{99.66}, \phantom{100.00}\llap{100.00}}} \\
osic-pulmonary-fibrosis-progression 
& \shortstack{\phantom{100.00}\llap{13.85}\\{\scriptsize \phantom{100.00}\llap{12.40}, \phantom{100.00}\llap{15.70}}} 
& \shortstack{\phantom{100.00}\llap{10.89}\\{\scriptsize \phantom{100.00}\llap{10.07}, \phantom{100.00}\llap{11.50}}} 
& \shortstack{\phantom{100.00}\llap{11.02}\\{\scriptsize \phantom{100.00}\llap{10.07}, \phantom{100.00}\llap{11.88}}} 
& \shortstack{\phantom{100.00}\llap{12.50}\\{\scriptsize \phantom{100.00}\llap{11.88}, \phantom{100.00}\llap{13.22}}} \\
petfinder-pawpularity-score 
& \shortstack{\phantom{100.00}\llap{54.37}\\{\scriptsize \phantom{100.00}\llap{52.78}, \phantom{100.00}\llap{55.75}}} 
& \shortstack{\phantom{100.00}\llap{56.41}\\{\scriptsize \phantom{100.00}\llap{54.11}, \phantom{100.00}\llap{59.88}}} 
& \shortstack{\phantom{100.00}\llap{63.37}\\{\scriptsize \phantom{100.00}\llap{51.12}, \phantom{100.00}\llap{83.74}}} 
& \shortstack{\phantom{100.00}\llap{88.62}\\{\scriptsize \phantom{100.00}\llap{68.45}, \phantom{100.00}\llap{99.72}}} \\
plant-pathology-2021-fgvc8 
& \shortstack{\phantom{100.00}\llap{100.00}\\{\scriptsize \phantom{100.00}\llap{100.00}, \phantom{100.00}\llap{100.00}}} 
& \shortstack{\phantom{100.00}\llap{100.00}\\{\scriptsize \phantom{100.00}\llap{100.00}, \phantom{100.00}\llap{100.00}}} 
& \shortstack{\phantom{100.00}\llap{100.00}\\{\scriptsize \phantom{100.00}\llap{100.00}, \phantom{100.00}\llap{100.00}}} 
& \shortstack{\phantom{100.00}\llap{100.00}\\{\scriptsize \phantom{100.00}\llap{100.00}, \phantom{100.00}\llap{100.00}}} \\
smartphone-decimeter-2022 
& \shortstack{\phantom{100.00}\llap{4.89}\\{\scriptsize \phantom{100.00}\llap{4.89}, \phantom{100.00}\llap{4.89}}} 
& \shortstack{\phantom{100.00}\llap{5.00}\\{\scriptsize \phantom{100.00}\llap{4.89}, \phantom{100.00}\llap{5.06}}} 
& \shortstack{\phantom{100.00}\llap{5.00}\\{\scriptsize \phantom{100.00}\llap{4.89}, \phantom{100.00}\llap{5.06}}} 
& \shortstack{\phantom{100.00}\llap{9.02}\\{\scriptsize \phantom{100.00}\llap{4.89}, \phantom{100.00}\llap{16.75}}} \\
spooky-author-identification 
& \shortstack{\phantom{100.00}\llap{97.29}\\{\scriptsize \phantom{100.00}\llap{96.45}, \phantom{100.00}\llap{98.15}}} 
& \shortstack{\phantom{100.00}\llap{98.04}\\{\scriptsize \phantom{100.00}\llap{97.99}, \phantom{100.00}\llap{98.07}}} 
& \shortstack{\phantom{100.00}\llap{97.90}\\{\scriptsize \phantom{100.00}\llap{97.26}, \phantom{100.00}\llap{98.31}}} 
& \shortstack{\phantom{100.00}\llap{98.39}\\{\scriptsize \phantom{100.00}\llap{98.31}, \phantom{100.00}\llap{98.47}}} \\
stanford-covid-vaccine 
& \shortstack{\phantom{100.00}\llap{100.00}\\{\scriptsize \phantom{100.00}\llap{100.00}, \phantom{100.00}\llap{100.00}}} 
& \shortstack{\phantom{100.00}\llap{100.00}\\{\scriptsize \phantom{100.00}\llap{100.00}, \phantom{100.00}\llap{100.00}}} 
& \shortstack{\phantom{100.00}\llap{100.00}\\{\scriptsize \phantom{100.00}\llap{100.00}, \phantom{100.00}\llap{100.00}}} 
& \shortstack{\phantom{100.00}\llap{100.00}\\{\scriptsize \phantom{100.00}\llap{100.00}, \phantom{100.00}\llap{100.00}}} \\
tensorflow2-question-answering 
& \shortstack{\phantom{100.00}\llap{48.91}\\{\scriptsize \phantom{100.00}\llap{48.66}, \phantom{100.00}\llap{49.15}}} 
& \shortstack{\phantom{100.00}\llap{94.30}\\{\scriptsize \phantom{100.00}\llap{88.97}, \phantom{100.00}\llap{97.89}}} 
& \shortstack{\phantom{100.00}\llap{97.27}\\{\scriptsize \phantom{100.00}\llap{96.51}, \phantom{100.00}\llap{97.89}}} 
& \shortstack{\phantom{100.00}\llap{97.73}\\{\scriptsize \phantom{100.00}\llap{96.51}, \phantom{100.00}\llap{98.54}}} \\
tweet-sentiment-extraction 
& \shortstack{\phantom{100.00}\llap{73.34}\\{\scriptsize \phantom{100.00}\llap{60.39}, \phantom{100.00}\llap{83.12}}} 
& \shortstack{\phantom{100.00}\llap{96.28}\\{\scriptsize \phantom{100.00}\llap{95.95}, \phantom{100.00}\llap{96.49}}} 
& \shortstack{\phantom{100.00}\llap{95.92}\\{\scriptsize \phantom{100.00}\llap{93.57}, \phantom{100.00}\llap{98.25}}} 
& \shortstack{\phantom{100.00}\llap{93.73}\\{\scriptsize \phantom{100.00}\llap{85.86}, \phantom{100.00}\llap{98.16}}} \\
us-patent-phrase-to-phrase-matching 
& \shortstack{\phantom{100.00}\llap{96.52}\\{\scriptsize \phantom{100.00}\llap{91.53}, \phantom{100.00}\llap{99.15}}} 
& \shortstack{\phantom{100.00}\llap{99.51}\\{\scriptsize \phantom{100.00}\llap{99.21}, \phantom{100.00}\llap{99.89}}} 
& \shortstack{\phantom{100.00}\llap{99.63}\\{\scriptsize \phantom{100.00}\llap{99.42}, \phantom{100.00}\llap{100.00}}} 
& \shortstack{\phantom{100.00}\llap{99.91}\\{\scriptsize \phantom{100.00}\llap{99.79}, \phantom{100.00}\llap{100.00}}} \\
uw-madison-gi-tract-image-segmentation 
& \shortstack{\phantom{100.00}\llap{18.06}\\{\scriptsize \phantom{100.00}\llap{4.91}, \phantom{100.00}\llap{43.83}}} 
& \shortstack{\phantom{100.00}\llap{22.54}\\{\scriptsize \phantom{100.00}\llap{5.30}, \phantom{100.00}\llap{56.43}}} 
& \shortstack{\phantom{100.00}\llap{26.18}\\{\scriptsize \phantom{100.00}\llap{5.04}, \phantom{100.00}\llap{67.61}}} 
& \shortstack{\phantom{100.00}\llap{33.35}\\{\scriptsize \phantom{100.00}\llap{5.04}, \phantom{100.00}\llap{89.08}}} \\
ventilator-pressure-prediction 
& \shortstack{\phantom{100.00}\llap{48.78}\\{\scriptsize \phantom{100.00}\llap{47.96}, \phantom{100.00}\llap{49.39}}} 
& \shortstack{\phantom{100.00}\llap{55.40}\\{\scriptsize \phantom{100.00}\llap{50.02}, \phantom{100.00}\llap{63.52}}} 
& \shortstack{\phantom{100.00}\llap{55.93}\\{\scriptsize \phantom{100.00}\llap{50.38}, \phantom{100.00}\llap{63.52}}} 
& \shortstack{\phantom{100.00}\llap{61.78}\\{\scriptsize \phantom{100.00}\llap{60.83}, \phantom{100.00}\llap{63.52}}} \\
whale-categorization-playground 
& \shortstack{\phantom{100.00}\llap{92.11}\\{\scriptsize \phantom{100.00}\llap{86.17}, \phantom{100.00}\llap{96.59}}} 
& \shortstack{\phantom{100.00}\llap{98.55}\\{\scriptsize \phantom{100.00}\llap{97.35}, \phantom{100.00}\llap{99.24}}} 
& \shortstack{\phantom{100.00}\llap{98.55}\\{\scriptsize \phantom{100.00}\llap{97.35}, \phantom{100.00}\llap{99.24}}} 
& \shortstack{\phantom{100.00}\llap{99.31}\\{\scriptsize \phantom{100.00}\llap{98.30}, \phantom{100.00}\llap{99.81}}} \\
\bottomrule
\end{tabularx}
\end{table}

\begin{table}[H]
\caption{Per-task test percentile scores for \atlasplus at different time cutoffs. Values shown as score with 95\% CI below.}
\label{tab:atlasplus_per_task_results}
\begin{tabularx}{\textwidth}{lcccc}
\toprule
Task & 3h & 12h & 24h & 72h \\
\midrule
aptos2019-blindness-detection 
& \shortstack{\phantom{100.00}\llap{92.59}\\{\scriptsize \phantom{100.00}\llap{83.67}, \phantom{100.00}\llap{98.22}}} 
& \shortstack{\phantom{100.00}\llap{95.75}\\{\scriptsize \phantom{100.00}\llap{88.70}, \phantom{100.00}\llap{99.39}}} 
& \shortstack{\phantom{100.00}\llap{95.75}\\{\scriptsize \phantom{100.00}\llap{88.70}, \phantom{100.00}\llap{99.39}}} 
& \shortstack{\phantom{100.00}\llap{96.60}\\{\scriptsize \phantom{100.00}\llap{93.03}, \phantom{100.00}\llap{99.25}}} \\
billion-word-imputation 
& \shortstack{\phantom{100.00}\llap{97.70}\\{\scriptsize \phantom{100.00}\llap{93.10}, \phantom{100.00}\llap{100.00}}} 
& \shortstack{\phantom{100.00}\llap{100.00}\\{\scriptsize \phantom{100.00}\llap{100.00}, \phantom{100.00}\llap{100.00}}} 
& \shortstack{\phantom{100.00}\llap{100.00}\\{\scriptsize \phantom{100.00}\llap{100.00}, \phantom{100.00}\llap{100.00}}} 
& \shortstack{\phantom{100.00}\llap{100.00}\\{\scriptsize \phantom{100.00}\llap{100.00}, \phantom{100.00}\llap{100.00}}} \\
bms-molecular-translation 
& \shortstack{\phantom{100.00}\llap{16.40}\\{\scriptsize \phantom{100.00}\llap{0.00}, \phantom{100.00}\llap{49.20}}} 
& \shortstack{\phantom{100.00}\llap{80.97}\\{\scriptsize \phantom{100.00}\llap{73.68}, \phantom{100.00}\llap{85.13}}} 
& \shortstack{\phantom{100.00}\llap{86.08}\\{\scriptsize \phantom{100.00}\llap{84.44}, \phantom{100.00}\llap{87.41}}} 
& \shortstack{\phantom{100.00}\llap{93.63}\\{\scriptsize \phantom{100.00}\llap{92.56}, \phantom{100.00}\llap{95.08}}} \\
cassava-leaf-disease-classification 
& \shortstack{\phantom{100.00}\llap{87.98}\\{\scriptsize \phantom{100.00}\llap{82.90}, \phantom{100.00}\llap{94.69}}} 
& \shortstack{\phantom{100.00}\llap{88.40}\\{\scriptsize \phantom{100.00}\llap{82.90}, \phantom{100.00}\llap{95.95}}} 
& \shortstack{\phantom{100.00}\llap{96.60}\\{\scriptsize \phantom{100.00}\llap{90.26}, \phantom{100.00}\llap{99.95}}} 
& \shortstack{\phantom{100.00}\llap{99.06}\\{\scriptsize \phantom{100.00}\llap{97.26}, \phantom{100.00}\llap{99.97}}} \\
champs-scalar-coupling 
& \shortstack{\phantom{100.00}\llap{95.05}\\{\scriptsize \phantom{100.00}\llap{92.70}, \phantom{100.00}\llap{96.38}}} 
& \shortstack{\phantom{100.00}\llap{99.06}\\{\scriptsize \phantom{100.00}\llap{98.69}, \phantom{100.00}\llap{99.31}}} 
& \shortstack{\phantom{100.00}\llap{99.27}\\{\scriptsize \phantom{100.00}\llap{99.01}, \phantom{100.00}\llap{99.45}}} 
& \shortstack{\phantom{100.00}\llap{99.53}\\{\scriptsize \phantom{100.00}\llap{99.49}, \phantom{100.00}\llap{99.56}}} \\
freesound-audio-tagging-2019 
& \shortstack{\phantom{100.00}\llap{95.84}\\{\scriptsize \phantom{100.00}\llap{93.05}, \phantom{100.00}\llap{97.84}}} 
& \shortstack{\phantom{100.00}\llap{99.36}\\{\scriptsize \phantom{100.00}\llap{98.08}, \phantom{100.00}\llap{100.00}}} 
& \shortstack{\phantom{100.00}\llap{100.00}\\{\scriptsize \phantom{100.00}\llap{100.00}, \phantom{100.00}\llap{100.00}}} 
& \shortstack{\phantom{100.00}\llap{100.00}\\{\scriptsize \phantom{100.00}\llap{100.00}, \phantom{100.00}\llap{100.00}}} \\
h-and-m-personalized-fashion-recommendations 
& \shortstack{\phantom{100.00}\llap{29.03}\\{\scriptsize \phantom{100.00}\llap{28.95}, \phantom{100.00}\llap{29.19}}} 
& \shortstack{\phantom{100.00}\llap{30.82}\\{\scriptsize \phantom{100.00}\llap{29.22}, \phantom{100.00}\llap{31.83}}} 
& \shortstack{\phantom{100.00}\llap{33.53}\\{\scriptsize \phantom{100.00}\llap{29.46}, \phantom{100.00}\llap{36.40}}} 
& \shortstack{\phantom{100.00}\llap{34.07}\\{\scriptsize \phantom{100.00}\llap{29.22}, \phantom{100.00}\llap{36.57}}} \\
hms-harmful-brain-activity-classification 
& \shortstack{\phantom{100.00}\llap{24.21}\\{\scriptsize \phantom{100.00}\llap{21.04}, \phantom{100.00}\llap{26.32}}} 
& \shortstack{\phantom{100.00}\llap{27.90}\\{\scriptsize \phantom{100.00}\llap{24.37}, \phantom{100.00}\llap{31.31}}} 
& \shortstack{\phantom{100.00}\llap{28.48}\\{\scriptsize \phantom{100.00}\llap{26.68}, \phantom{100.00}\llap{30.73}}} 
& \shortstack{\phantom{100.00}\llap{28.21}\\{\scriptsize \phantom{100.00}\llap{26.46}, \phantom{100.00}\llap{31.49}}} \\
hotel-id-2021-fgvc8 
& \shortstack{\phantom{100.00}\llap{91.30}\\{\scriptsize \phantom{100.00}\llap{90.22}, \phantom{100.00}\llap{92.39}}} 
& \shortstack{\phantom{100.00}\llap{95.29}\\{\scriptsize \phantom{100.00}\llap{93.48}, \phantom{100.00}\llap{96.74}}} 
& \shortstack{\phantom{100.00}\llap{97.10}\\{\scriptsize \phantom{100.00}\llap{96.74}, \phantom{100.00}\llap{97.83}}} 
& \shortstack{\phantom{100.00}\llap{98.55}\\{\scriptsize \phantom{100.00}\llap{96.74}, \phantom{100.00}\llap{100.00}}} \\
hubmap-kidney-segmentation 
& \shortstack{\phantom{100.00}\llap{89.27}\\{\scriptsize \phantom{100.00}\llap{85.15}, \phantom{100.00}\llap{92.16}}} 
& \shortstack{\phantom{100.00}\llap{95.25}\\{\scriptsize \phantom{100.00}\llap{89.41}, \phantom{100.00}\llap{98.42}}} 
& \shortstack{\phantom{100.00}\llap{96.19}\\{\scriptsize \phantom{100.00}\llap{92.58}, \phantom{100.00}\llap{99.67}}} 
& \shortstack{\phantom{100.00}\llap{98.25}\\{\scriptsize \phantom{100.00}\llap{97.50}, \phantom{100.00}\llap{99.67}}} \\
imet-2020-fgvc7 
& \shortstack{\phantom{100.00}\llap{64.21}\\{\scriptsize \phantom{100.00}\llap{53.68}, \phantom{100.00}\llap{83.16}}} 
& \shortstack{\phantom{100.00}\llap{85.26}\\{\scriptsize \phantom{100.00}\llap{85.26}, \phantom{100.00}\llap{85.26}}} 
& \shortstack{\phantom{100.00}\llap{87.72}\\{\scriptsize \phantom{100.00}\llap{85.26}, \phantom{100.00}\llap{90.53}}} 
& \shortstack{\phantom{100.00}\llap{92.63}\\{\scriptsize \phantom{100.00}\llap{91.58}, \phantom{100.00}\llap{94.74}}} \\
jigsaw-unintended-bias-in-toxicity-classification 
& \shortstack{\phantom{100.00}\llap{7.81}\\{\scriptsize \phantom{100.00}\llap{7.75}, \phantom{100.00}\llap{7.90}}} 
& \shortstack{\phantom{100.00}\llap{8.20}\\{\scriptsize \phantom{100.00}\llap{7.90}, \phantom{100.00}\llap{8.70}}} 
& \shortstack{\phantom{100.00}\llap{8.34}\\{\scriptsize \phantom{100.00}\llap{7.98}, \phantom{100.00}\llap{8.70}}} 
& \shortstack{\phantom{100.00}\llap{8.66}\\{\scriptsize \phantom{100.00}\llap{8.36}, \phantom{100.00}\llap{9.23}}} \\
kuzushiji-recognition 
& \shortstack{\phantom{100.00}\llap{99.54}\\{\scriptsize \phantom{100.00}\llap{98.63}, \phantom{100.00}\llap{100.00}}} 
& \shortstack{\phantom{100.00}\llap{100.00}\\{\scriptsize \phantom{100.00}\llap{100.00}, \phantom{100.00}\llap{100.00}}} 
& \shortstack{\phantom{100.00}\llap{100.00}\\{\scriptsize \phantom{100.00}\llap{100.00}, \phantom{100.00}\llap{100.00}}} 
& \shortstack{\phantom{100.00}\llap{100.00}\\{\scriptsize \phantom{100.00}\llap{100.00}, \phantom{100.00}\llap{100.00}}} \\
mlsp-2013-birds 
& \shortstack{\phantom{100.00}\llap{95.00}\\{\scriptsize \phantom{100.00}\llap{90.00}, \phantom{100.00}\llap{98.75}}} 
& \shortstack{\phantom{100.00}\llap{96.67}\\{\scriptsize \phantom{100.00}\llap{93.75}, \phantom{100.00}\llap{98.75}}} 
& \shortstack{\phantom{100.00}\llap{99.17}\\{\scriptsize \phantom{100.00}\llap{98.75}, \phantom{100.00}\llap{100.00}}} 
& \shortstack{\phantom{100.00}\llap{99.17}\\{\scriptsize \phantom{100.00}\llap{98.75}, \phantom{100.00}\llap{100.00}}} \\
multi-modal-gesture-recognition 
& \shortstack{\phantom{100.00}\llap{100.00}\\{\scriptsize \phantom{100.00}\llap{100.00}, \phantom{100.00}\llap{100.00}}} 
& \shortstack{\phantom{100.00}\llap{100.00}\\{\scriptsize \phantom{100.00}\llap{100.00}, \phantom{100.00}\llap{100.00}}} 
& \shortstack{\phantom{100.00}\llap{100.00}\\{\scriptsize \phantom{100.00}\llap{100.00}, \phantom{100.00}\llap{100.00}}} 
& \shortstack{\phantom{100.00}\llap{100.00}\\{\scriptsize \phantom{100.00}\llap{100.00}, \phantom{100.00}\llap{100.00}}} \\
new-york-city-taxi-fare-prediction 
& \shortstack{\phantom{100.00}\llap{39.71}\\{\scriptsize \phantom{100.00}\llap{36.05}, \phantom{100.00}\llap{42.45}}} 
& \shortstack{\phantom{100.00}\llap{38.54}\\{\scriptsize \phantom{100.00}\llap{36.05}, \phantom{100.00}\llap{41.78}}} 
& \shortstack{\phantom{100.00}\llap{36.97}\\{\scriptsize \phantom{100.00}\llap{36.05}, \phantom{100.00}\llap{37.80}}} 
& \shortstack{\phantom{100.00}\llap{38.63}\\{\scriptsize \phantom{100.00}\llap{36.66}, \phantom{100.00}\llap{40.90}}} \\
nfl-player-contact-detection 
& \shortstack{\phantom{100.00}\llap{92.08}\\{\scriptsize \phantom{100.00}\llap{88.29}, \phantom{100.00}\llap{94.46}}} 
& \shortstack{\phantom{100.00}\llap{95.03}\\{\scriptsize \phantom{100.00}\llap{94.14}, \phantom{100.00}\llap{95.95}}} 
& \shortstack{\phantom{100.00}\llap{95.56}\\{\scriptsize \phantom{100.00}\llap{93.82}, \phantom{100.00}\llap{96.81}}} 
& \shortstack{\phantom{100.00}\llap{96.77}\\{\scriptsize \phantom{100.00}\llap{96.06}, \phantom{100.00}\llap{97.12}}} \\
nomad2018-predict-transparent-conductors 
& \shortstack{\phantom{100.00}\llap{100.00}\\{\scriptsize \phantom{100.00}\llap{100.00}, \phantom{100.00}\llap{100.00}}} 
& \shortstack{\phantom{100.00}\llap{100.00}\\{\scriptsize \phantom{100.00}\llap{100.00}, \phantom{100.00}\llap{100.00}}} 
& \shortstack{\phantom{100.00}\llap{100.00}\\{\scriptsize \phantom{100.00}\llap{100.00}, \phantom{100.00}\llap{100.00}}} 
& \shortstack{\phantom{100.00}\llap{100.00}\\{\scriptsize \phantom{100.00}\llap{100.00}, \phantom{100.00}\llap{100.00}}} \\
osic-pulmonary-fibrosis-progression 
& \shortstack{\phantom{100.00}\llap{12.56}\\{\scriptsize \phantom{100.00}\llap{12.31}, \phantom{100.00}\llap{12.88}}} 
& \shortstack{\phantom{100.00}\llap{12.53}\\{\scriptsize \phantom{100.00}\llap{11.12}, \phantom{100.00}\llap{14.17}}} 
& \shortstack{\phantom{100.00}\llap{12.02}\\{\scriptsize \phantom{100.00}\llap{11.12}, \phantom{100.00}\llap{13.07}}} 
& \shortstack{\phantom{100.00}\llap{12.34}\\{\scriptsize \phantom{100.00}\llap{12.17}, \phantom{100.00}\llap{12.45}}} \\
petfinder-pawpularity-score 
& \shortstack{\phantom{100.00}\llap{75.26}\\{\scriptsize \phantom{100.00}\llap{58.75}, \phantom{100.00}\llap{99.46}}} 
& \shortstack{\phantom{100.00}\llap{94.19}\\{\scriptsize \phantom{100.00}\llap{93.13}, \phantom{100.00}\llap{96.21}}} 
& \shortstack{\phantom{100.00}\llap{98.13}\\{\scriptsize \phantom{100.00}\llap{96.10}, \phantom{100.00}\llap{100.00}}} 
& \shortstack{\phantom{100.00}\llap{98.70}\\{\scriptsize \phantom{100.00}\llap{96.10}, \phantom{100.00}\llap{100.00}}} \\
plant-pathology-2021-fgvc8 
& \shortstack{\phantom{100.00}\llap{100.00}\\{\scriptsize \phantom{100.00}\llap{100.00}, \phantom{100.00}\llap{100.00}}} 
& \shortstack{\phantom{100.00}\llap{100.00}\\{\scriptsize \phantom{100.00}\llap{100.00}, \phantom{100.00}\llap{100.00}}} 
& \shortstack{\phantom{100.00}\llap{100.00}\\{\scriptsize \phantom{100.00}\llap{100.00}, \phantom{100.00}\llap{100.00}}} 
& \shortstack{\phantom{100.00}\llap{100.00}\\{\scriptsize \phantom{100.00}\llap{100.00}, \phantom{100.00}\llap{100.00}}} \\
smartphone-decimeter-2022 
& \shortstack{\phantom{100.00}\llap{8.96}\\{\scriptsize \phantom{100.00}\llap{4.89}, \phantom{100.00}\llap{16.75}}} 
& \shortstack{\phantom{100.00}\llap{40.78}\\{\scriptsize \phantom{100.00}\llap{4.89}, \phantom{100.00}\llap{100.00}}} 
& \shortstack{\phantom{100.00}\llap{36.77}\\{\scriptsize \phantom{100.00}\llap{4.89}, \phantom{100.00}\llap{100.00}}} 
& \shortstack{\phantom{100.00}\llap{36.59}\\{\scriptsize \phantom{100.00}\llap{4.89}, \phantom{100.00}\llap{100.00}}} \\
spooky-author-identification 
& \shortstack{\phantom{100.00}\llap{98.25}\\{\scriptsize \phantom{100.00}\llap{98.15}, \phantom{100.00}\llap{98.31}}} 
& \shortstack{\phantom{100.00}\llap{98.74}\\{\scriptsize \phantom{100.00}\llap{98.55}, \phantom{100.00}\llap{98.95}}} 
& \shortstack{\phantom{100.00}\llap{98.93}\\{\scriptsize \phantom{100.00}\llap{98.87}, \phantom{100.00}\llap{98.95}}} 
& \shortstack{\phantom{100.00}\llap{99.03}\\{\scriptsize \phantom{100.00}\llap{98.95}, \phantom{100.00}\llap{99.11}}} \\
stanford-covid-vaccine 
& \shortstack{\phantom{100.00}\llap{100.00}\\{\scriptsize \phantom{100.00}\llap{100.00}, \phantom{100.00}\llap{100.00}}} 
& \shortstack{\phantom{100.00}\llap{100.00}\\{\scriptsize \phantom{100.00}\llap{100.00}, \phantom{100.00}\llap{100.00}}} 
& \shortstack{\phantom{100.00}\llap{100.00}\\{\scriptsize \phantom{100.00}\llap{100.00}, \phantom{100.00}\llap{100.00}}} 
& \shortstack{\phantom{100.00}\llap{100.00}\\{\scriptsize \phantom{100.00}\llap{100.00}, \phantom{100.00}\llap{100.00}}} \\
tensorflow2-question-answering 
& \shortstack{\phantom{100.00}\llap{48.91}\\{\scriptsize \phantom{100.00}\llap{48.66}, \phantom{100.00}\llap{49.07}}} 
& \shortstack{\phantom{100.00}\llap{80.35}\\{\scriptsize \phantom{100.00}\llap{49.47}, \phantom{100.00}\llap{97.40}}} 
& \shortstack{\phantom{100.00}\llap{95.59}\\{\scriptsize \phantom{100.00}\llap{92.62}, \phantom{100.00}\llap{97.97}}} 
& \shortstack{\phantom{100.00}\llap{98.54}\\{\scriptsize \phantom{100.00}\llap{97.89}, \phantom{100.00}\llap{99.43}}} \\
tweet-sentiment-extraction 
& \shortstack{\phantom{100.00}\llap{88.13}\\{\scriptsize \phantom{100.00}\llap{79.00}, \phantom{100.00}\llap{96.31}}} 
& \shortstack{\phantom{100.00}\llap{97.93}\\{\scriptsize \phantom{100.00}\llap{96.90}, \phantom{100.00}\llap{98.52}}} 
& \shortstack{\phantom{100.00}\llap{98.53}\\{\scriptsize \phantom{100.00}\llap{98.20}, \phantom{100.00}\llap{99.10}}} 
& \shortstack{\phantom{100.00}\llap{98.25}\\{\scriptsize \phantom{100.00}\llap{97.44}, \phantom{100.00}\llap{99.10}}} \\
us-patent-phrase-to-phrase-matching 
& \shortstack{\phantom{100.00}\llap{99.51}\\{\scriptsize \phantom{100.00}\llap{99.42}, \phantom{100.00}\llap{99.68}}} 
& \shortstack{\phantom{100.00}\llap{99.93}\\{\scriptsize \phantom{100.00}\llap{99.84}, \phantom{100.00}\llap{100.00}}} 
& \shortstack{\phantom{100.00}\llap{100.00}\\{\scriptsize \phantom{100.00}\llap{100.00}, \phantom{100.00}\llap{100.00}}} 
& \shortstack{\phantom{100.00}\llap{100.00}\\{\scriptsize \phantom{100.00}\llap{100.00}, \phantom{100.00}\llap{100.00}}} \\
uw-madison-gi-tract-image-segmentation 
& \shortstack{\phantom{100.00}\llap{61.39}\\{\scriptsize \phantom{100.00}\llap{57.27}, \phantom{100.00}\llap{68.07}}} 
& \shortstack{\phantom{100.00}\llap{81.02}\\{\scriptsize \phantom{100.00}\llap{59.21}, \phantom{100.00}\llap{93.21}}} 
& \shortstack{\phantom{100.00}\llap{87.90}\\{\scriptsize \phantom{100.00}\llap{80.38}, \phantom{100.00}\llap{93.21}}} 
& \shortstack{\phantom{100.00}\llap{93.95}\\{\scriptsize \phantom{100.00}\llap{89.40}, \phantom{100.00}\llap{97.29}}} \\
ventilator-pressure-prediction 
& \shortstack{\phantom{100.00}\llap{46.17}\\{\scriptsize \phantom{100.00}\llap{44.62}, \phantom{100.00}\llap{47.20}}} 
& \shortstack{\phantom{100.00}\llap{53.94}\\{\scriptsize \phantom{100.00}\llap{52.34}, \phantom{100.00}\llap{55.76}}} 
& \shortstack{\phantom{100.00}\llap{58.74}\\{\scriptsize \phantom{100.00}\llap{55.57}, \phantom{100.00}\llap{61.18}}} 
& \shortstack{\phantom{100.00}\llap{70.06}\\{\scriptsize \phantom{100.00}\llap{60.48}, \phantom{100.00}\llap{88.10}}} \\
whale-categorization-playground 
& \shortstack{\phantom{100.00}\llap{94.63}\\{\scriptsize \phantom{100.00}\llap{91.86}, \phantom{100.00}\llap{96.59}}} 
& \shortstack{\phantom{100.00}\llap{96.84}\\{\scriptsize \phantom{100.00}\llap{96.21}, \phantom{100.00}\llap{97.35}}} 
& \shortstack{\phantom{100.00}\llap{98.36}\\{\scriptsize \phantom{100.00}\llap{97.35}, \phantom{100.00}\llap{99.43}}} 
& \shortstack{\phantom{100.00}\llap{99.56}\\{\scriptsize \phantom{100.00}\llap{99.43}, \phantom{100.00}\llap{99.81}}} \\
\bottomrule
\end{tabularx}
\end{table}

\newpage
\section{AIRS-Bench}
\label{airsbench-appendix}

\begin{table}[H]
\centering
\caption{Per-task results on AIRS-Bench. Each task group shows \atlas followed by the two best-performing baselines (selected by best seed). Normalised scores are scaled so that 0\,=\,worst and 100\,=\,SOTA. $\infty$ denotes normalised scores whereby 1 seed achieved the optimal score. This is usually due to unfair solutions. See Section \ref{sec:airs_bench}. We show the best baseline per task selected by the best seed score since we care about SOTA beating solutions. \textbf{Bold} indicates the best method per task (by mean normalised score).}
\label{tab:airs_results}
\small
\begin{tabularx}{\textwidth}{l l r r r r r r r}
\toprule
\textbf{Task} & \textbf{Method} & \textbf{SOTA} & \textbf{Best} & \textbf{Mean} & \textbf{\shortstack{Norm. \\Best}} & \textbf{\shortstack{Norm. \\Mean}} & \textbf{Seeds} & \textbf{\shortstack{Beats\\SOTA}} \\
\midrule
Code Gen. & \textbf{\atlas} & 0.1870 & \textbf{0.2810} & \textbf{0.2309} & \textbf{159.4} & \textbf{127.3} & \textbf{3} & \textbf{3} \\
 & $AIRA_{\text{greedy}}$ GPT5 &  & 0.0050 & 0.0020 & 2.4 & 1.0 & 10 & 0 \\
\addlinespace[4pt]
Code Retrieval & \textbf{\atlas} & 0.6113 & \textbf{0.5772} & \textbf{0.5037} & \textbf{91.1} & \textbf{74.8} & \textbf{3} & \textbf{0} \\
 & $AIRA_{\text{greedy}}$ GPT5 &  & 0.4485 & 0.1968 & 63.0 & 23.2 & 10 & 0 \\
\addlinespace[4pt]
Coref. (WSC) & \textbf{\atlas} & 0.9620 & \textbf{1.0000} & \textbf{0.9952} & $\infty$ & $\infty$ & \textbf{4} & \textbf{4} \\
 & $AIRA_{\text{greedy}}$ GPT5 &  & 0.9423 & 0.7577 & 85.2 & 34.2 & 10 & 0 \\
\addlinespace[4pt]
Coref. (Wino.) & \textbf{\atlas} & 0.8540 & \textbf{0.9132} & \textbf{0.9092} & \textbf{130.4} & \textbf{127.9} & \textbf{4} & \textbf{4} \\
 & $AIRA_{\text{greedy}}$ GPT5 &  & 0.9148 & 0.6938 & 131.5 & 56.6 & 10 & 4 \\
\addlinespace[4pt]
Mol. Prop. ($C_v$) & \textbf{\atlas} & 0.0210 & \textbf{0.0193} & \textbf{0.0199} & \textbf{100.8} & \textbf{100.5} & \textbf{4} & \textbf{4} \\
 & $AIRA_{\text{greedy}}$ GPT5 &  & 0.0306 & 0.0557 & 96.4 & 90.8 & 10 & 0 \\
\addlinespace[4pt]
Mol. Prop. ($G$) & \textbf{\atlas} & 7.5300 & \textbf{5.5464} & \textbf{5.6387} & \textbf{101.1} & \textbf{101.0} & \textbf{4} & \textbf{4} \\
 & $AIRA_{\text{greedy}}$ o3 &  & 17.6286 & $3.39 \times 10^{4}$ & 97.0 & 70.0 & 10 & 0 \\
\addlinespace[4pt]
Graph Reg. & \textbf{\atlas} & 0.0170 & \textbf{0.0226} & \textbf{0.0227} & \textbf{93.9} & \textbf{93.7} & \textbf{4} & \textbf{0} \\
 & $AIRA_{\text{greedy}}$ o3 &  & 0.0408 & 0.0984 & 81.1 & 62.0 & 10 & 0 \\
\addlinespace[4pt]
Math QA & \textbf{\atlas} & 0.9420 & \textbf{0.9367} & \textbf{0.9267} & \textbf{96.9} & \textbf{92.2} & \textbf{4} & \textbf{0} \\
 & $AIRA_{\text{greedy}}$ o3 &  & 0.6500 & 0.4253 & 36.9 & 19.5 & 10 & 0 \\
\addlinespace[4pt]
QA (DuoRC) & \textbf{\atlas} & 0.4648 & \textbf{0.3925} & \textbf{0.3842} & \textbf{79.7} & \textbf{77.6} & \textbf{4} & \textbf{0} \\
 & $AIRA_{\text{oneshot}}$ OSS-120B &  & 0.3262 & 0.2782 & 63.2 & 52.2 & 20 & 0 \\
\addlinespace[4pt]
QA (ELI5) & \textbf{\atlas} & 0.2690 & \textbf{0.2523} & \textbf{0.2483} & \textbf{92.7} & \textbf{91.0} & \textbf{4} & \textbf{0} \\
 & $AIRA_{\text{greedy}}$ GPT5 &  & 0.2354 & 0.1881 & 85.6 & 66.2 & 10 & 0 \\
\addlinespace[4pt]
QA (FinQA) & \textbf{\atlas} & 0.7803 & \textbf{1.0000} & \textbf{0.6713} & $\infty$ & $\infty$ & \textbf{4} & \textbf{1} \\
 & $AIRA_{\text{greedy}}$ GPT5 &  & 0.0924 & 0.0340 & 6.4 & 2.3 & 10 & 0 \\
\addlinespace[4pt]
Mol. Prop. ($R^2$) & \textbf{\atlas} & 0.0330 & \textbf{0.0221} & \textbf{0.0241} & \textbf{103.6} & \textbf{102.8} & \textbf{4} & \textbf{4} \\
 & $AIRA_{\text{greedy}}$ GPT5 &  & 0.1550 & 0.5619 & 86.2 & 74.7 & 10 & 0 \\
\addlinespace[4pt]
Reading Comp. & \textbf{\atlas} & 0.8580 & \textbf{0.8437} & \textbf{0.8316} & \textbf{95.1} & \textbf{91.4} & \textbf{4} & \textbf{0} \\
 & $AIRA_{\text{greedy}}$ GPT5 &  & 0.8645 & 0.5187 & 102.4 & 37.5 & 10 & 2 \\
\addlinespace[4pt]
Sentiment & \textbf{\atlas} & 0.7780 & \textbf{0.7311} & \textbf{0.7304} & \textbf{85.1} & \textbf{84.9} & \textbf{4} & \textbf{0} \\
 & $AIRA_{\text{greedy}}$ GPT5 &  & 0.7165 & 0.6671 & 80.9 & 68.4 & 10 & 0 \\
\addlinespace[4pt]
Text Cls. & \textbf{\atlas} & 0.9050 & \textbf{0.9333} & \textbf{0.9298} & \textbf{116.1} & \textbf{113.8} & \textbf{4} & \textbf{4} \\
 & $AIRA_{\text{greedy}}$ o3 &  & 0.9327 & 0.9012 & 115.7 & 98.2 & 10 & 9 \\
\addlinespace[4pt]
Text Sim. & \textbf{\atlas} & 0.8540 & \textbf{0.9083} & \textbf{0.9070} & \textbf{124.1} & \textbf{123.3} & \textbf{4} & \textbf{4} \\
 & $AIRA_{\text{greedy}}$ GPT5 &  & 0.9078 & 0.8737 & 123.8 & 107.5 & 10 & 8 \\
\addlinespace[4pt]
TS (Traffic) & \textbf{\atlas} & 0.6220 & $4.72 \times 10^{11}$ & $4.72 \times 10^{11}$ & \textbf{18.8} & \textbf{12.5} & \textbf{3} & \textbf{0} \\
 & $AIRA_{\text{greedy}}$ OSS-120B &  & $4.28 \times 10^{11}$ & $5.82 \times 10^{12}$ & 19.1 & 11.4 & 10 & 0 \\
\addlinespace[4pt]
TS (Rideshare) & \textbf{\atlas} & 1.1850 & \textbf{1.0661} & \textbf{1.1528} & \textbf{105.6} & \textbf{101.5} & \textbf{3} & \textbf{1} \\
 & $AIRA_{\text{greedy}}$ OSS-20B &  & 1.1741 & 1.2884 & 100.5 & 95.6 & 10 & 2 \\
\addlinespace[4pt]
TS (Solar) & \textbf{\atlas} & 576.35 & $8.06 \times 10^{2}$ & $1.06 \times 10^{3}$ & \textbf{83.7} & \textbf{72.0} & \textbf{3} & \textbf{0} \\
 & $AIRA_{\text{greedy}}$ GPT5 &  & $7.49 \times 10^{2}$ & $1.29 \times 10^{3}$ & 87.3 & 61.1 & 10 & 0 \\
\addlinespace[4pt]
Mol. Prop. ($U_0$) & \textbf{\atlas} & 5.83 & \textbf{4.4481} & \textbf{6.8336} & \textbf{101.0} & \textbf{99.9} & \textbf{4} & \textbf{3} \\
 & $AIRA_{\text{greedy}}$ GPT5 &  & 11.9779 & 77.3832 & 97.5 & 90.9 & 10 & 0 \\
\bottomrule
\end{tabularx}
\end{table}

\newpage

\section{Compute Optimal Resource Allocation}
\label{appendix:optimal_allocation}

Since $P = 100 \cdot g/(g+1)$ is strictly monotonically increasing in~$g$, and $g = \alpha \cdot h$ with $\alpha > 0$, maximizing $P$ over $(N, t)$ subject to a fixed compute budget $C = N \cdot t$ reduces to maximizing
\begin{equation}
    h(N, t) = \log(\gamma \, t + 1) \cdot \log(\beta \, N + 1).
\end{equation}

\begin{theorem}
\label{thm:optimal_allocation}
Given the compute budget constraint $C = N \cdot t$ with $C, \gamma, \beta > 0$, the function $h(N, t) = \log(\gamma \, t + 1) \cdot \log(\beta \, N + 1)$ achieves its unique global maximum at:
\begin{equation}
    N^* = \sqrt{\frac{\gamma C}{\beta}}, \qquad t^* = \sqrt{\frac{\beta C}{\gamma}}.
\end{equation}
\end{theorem}

\begin{proof}
Substituting $t = \frac{C}{N}$ into the objective, we obtain a single-variable function $h(N) = \log\left(\frac{\gamma \, C}{N} + 1\right) \cdot \log(\beta \, N + 1)$. To find the stationary points, we introduce auxiliary variables $x \triangleq \beta N$ and $y \triangleq \frac{\gamma C}{N}$. Note that $x \cdot y = \beta \gamma C$, which is a constant independent of $N$.

Setting the derivative of $h(N)$ with respect to $N$ to zero yields:
\begin{equation}
    \frac{d h}{d N} = \frac{\beta}{1 + \beta N} \log\left(1 + \frac{\gamma C}{N}\right) - \frac{\gamma C / N^2}{1 + \frac{\gamma C}{N}} \log(1 + \beta N) = 0.
\end{equation}

Substituting $x$, $y$, and the identities $\beta = \frac{x}{N}$ and $\frac{\gamma C}{N^2} = \frac{y}{N}$ into the stationarity condition, we have:
\begin{equation}
    \frac{x/N}{1 + x} \log(1 + y) - \frac{y/N}{1 + y} \log(1 + x) = 0.
\end{equation}

Multiplying by $N$ and separating the variables to opposite sides of the equation gives:
\begin{equation}
    \frac{1+x}{x} \log(1+x) = \frac{1+y}{y} \log(1+y).
\end{equation}

Observe that this equation is of the form $f(x) = f(y)$, where $f(z) \triangleq \frac{1+z}{z} \log(1+z)$. To determine the injectivity of $f(z)$ for $z > 0$, we evaluate its derivative:
\begin{equation}
    f'(z) = \frac{z - \log(1+z)}{z^2}.
\end{equation}
By the standard logarithmic inequality $z > \log(1+z)$ for all $z > 0$, it follows that $f'(z) > 0$. Therefore, $f(z)$ is strictly monotonically increasing, implying that $f(x) = f(y)$ if and only if $x = y$.

Equating $x = y$ and substituting their original definitions back yields:
\begin{align}
    \beta N &= \frac{\gamma C}{N} \nonumber \\
    N^* &= \sqrt{\frac{\gamma C}{\beta}}.
\end{align}
Applying the constraint $t = \frac{C}{N}$ gives the corresponding optimal time allocation $t^* = \sqrt{\frac{\beta C}{\gamma}}$, concluding the proof.
\end{proof}

\textbf{Remark (Optimal Performance Trajectory).} At the optimum, both logarithmic arguments equalize to $\sqrt{\beta \gamma C} + 1$, so $h^* = \left[\log\left(\sqrt{\beta \gamma C} + 1\right)\right]^2$ and $g^* = \alpha \cdot h^*$. Substituting into $P = 100 \cdot g/(g+1)$ gives the maximum achievable performance as a function of the total compute budget $C$:
\begin{equation}
\label{eq:optimal_p_c}
    P(C) = 100 \cdot \frac{\alpha \left[ \log\left(\sqrt{\beta \gamma C} + 1\right) \right]^2}{\alpha \left[ \log\left(\sqrt{\beta \gamma C} + 1\right) \right]^2 + 1}.
\end{equation}

\textbf{Remark (Integer Constraint).} Since $N$ must be a positive integer in practice, we round $N^*$ to the nearest integer, $N^* \leftarrow [ N^* ]$, and set $t^* = C / N^*$.

\end{document}